\documentclass{article}

\usepackage[final]{corl_2025} 

\usepackage{graphicx}
\usepackage{epsfig} 
\usepackage{mathptmx}
\usepackage{amsmath} 
\usepackage{wrapfig}

\usepackage{amssymb}  
\usepackage{amsthm}
\usepackage{wasysym}
\usepackage{mathrsfs}
\usepackage{cancel}
\usepackage[mathcal]{euscript}
\usepackage{bbm}
\usepackage{color}
\usepackage{lipsum}
\usepackage{algorithm}
\usepackage{algpseudocode}
\usepackage[normalem]{ulem}
\usepackage{tabularx}
\usepackage{booktabs}
\usepackage{array}

\usepackage{multirow}

\usepackage{subcaption}
\usepackage{graphicx}

\usepackage{xfrac}

\newcolumntype{Y}{>{\centering\arraybackslash}X}
\DeclareMathOperator*{\argmin}{arg\,min}

\usepackage{outlines}

\usepackage{enumitem}
\setenumerate[2]{label=\alph*.}
\setenumerate[3]{label=\roman*.}

\makeatletter
\patchcmd{\@makecaption}
  {\scshape}
  {}
  {}
  {}
\makeatother

\usepackage[capitalise]{cleveref}
\crefname{equation}{Eq.}{Eqs.}
\crefname{figure}{Fig.}{Figs.}
\crefname{tabular}{Tab.}{Tabs.}
\crefname{section}{Sec.}{Secs.}
\crefname{algorithm}{Alg.}{Algs.}
\crefname{appendix}{Appx.}{Appxs.}
\crefname{thm}{Thm.}{Thms.}
\crefname{cor}{Cor.}{Cors.}
\crefname{algorithm}{Alg.}{Algs.}
\crefname{lemma}{Lem.}{Lems.}
\crefname{problem}{Prob.}{Probs.}

\newcommand{\purple}[1]{{{\textcolor{purple}{#1}}}}
\definecolor{uncertifiedgreen}{rgb}{0.1,0.8,0.4}
\definecolor{purple}{rgb}{0.62, 0.12, 0.9}
\newcommand{\uncertified}[1]{{{\textcolor{uncertifiedgreen}{#1}}}}

\newtheorem{defn}{Definition}

\newtheorem{thm}[defn]{Theorem}
\newtheorem{cor}[defn]{Corollary}
\newtheorem{problem}[defn]{Problem}

\providecommand{\R}{\ensuremath \mathbb{R}}
\newcommand{\indicator}{\ensuremath \mathbf{1}}

\newcommand{\vc}[1]{#1}
\newcommand{\set}[1]{\mathcal{#1}}

\newcommand{\zeros}{\mathbf{0}}

\newcommand{\eye}{\vc{I}}

\newcommand{\inv}{^{-1}}

\newcommand{\expectation}{\mathbb{E}}

\newcommand{\sample}{\vc{x}}
\newcommand{\optvar}{\vc{k}}
\newcommand{\jointsample}{\vc{y}}
\newcommand{\state}{\vc{s}}
\newcommand{\action}{\vc{a}}

\newcommand{\observation}{\vc{o}}

\newcommand{\param}{\vc{k}}

\newcommand{\paramspace}{\set{K}}

\newcommand{\samplespace}{\set{X}}

\newcommand{\interactionfunc}{\vc{V}}

\newcommand{\alphasample}{\alpha^{\sample}}
\newcommand{\alphaoptvar}{\alpha^{\optvar}}
\newcommand{\alphajoint}{\alpha^{\jointsample}}

\newcommand{\prob}{\vc{p}}
\newcommand{\normal}{\mathcal{N}}

\newcommand{\objective}{J}

\newcommand{\ineqconstraint}{g}

\newcommand{\ours}{Ours}

\newcommand{\colvec}[2]{\begin{bmatrix} #1 \\ #2 \end{bmatrix}}

\newcommand{\nocontentsline}[3]{}
\newcommand{\tocless}[2]{\bgroup\let\addcontentsline=\nocontentsline#1{#2}\egroup}

\title{Joint Model-based Model-free Diffusion \\ 
for Planning with Constraints}

\author{
\bf{Wonsuhk Jung$^{*}$, Utkarsh A. Mishra$^{*}$, Nadun Ranawaka Arachchige,} \\
\bf{Yongxin Chen$^{\dagger}$, Danfei Xu$^{\dagger}$, Shreyas Kousik$^{\dagger}$} \\
Georgia Institute of Technology, Atlanta, GA, USA \\
$^{*}$ indicates equal contribution, $^{\dagger}$ indicates equal advising \\
\texttt{wonsuhk.jung@gatech.edu}
}

\begin{document}
\maketitle


\begin{abstract}
Model-free diffusion planners have shown great promise for robot motion planning, but practical robotic systems often require combining them with model-based optimization modules to enforce constraints, such as safety. 
Na\"ively integrating these modules presents compatibility challenges when diffusion's multi-modal outputs behave adversarially to optimization-based modules.
To address this, we introduce Joint Model-based Model-free Diffusion (JM2D), a novel generative modeling framework. 
JM2D formulates module integration as a joint sampling problem to maximize compatibility via an interaction potential, without additional training.
Using importance sampling, JM2D guides modules outputs based only on evaluations of the interaction potential, thus handling non-differentiable objectives commonly arising from non-convex optimization modules.
We evaluate JM2D via application to aligning diffusion planners with safety modules on offline RL and robot manipulation.
JM2D significantly improves task performance compared to conventional safety filters without sacrificing safety.
Further, we show that conditional generation is a special case of JM2D and elucidate key design choices by comparing with SOTA gradient-based and projection-based diffusion planners.
More details at: \url{https://jm2d-corl25.github.io/}
\end{abstract}


\keywords{Diffusion Models, Safety, Constrained Generation}


\vspace{-0.5em}
\section{Introduction}
\vspace{-0.5em}

Diffusion models have advanced generative modeling of diverse, high-quality samples in challenging domains, such as images, audio, and sequential data~\cite{song2020score, ho2020denoising}.
Encouraged by these successes, robotics research has increasingly leveraged diffusion for planning, control, and imitation learning~\cite{chi2023diffusion, janner2022planning, carvalho2024motion, pan2024model}, exploiting its ability to capture multimodal trajectories in a model-free way. 

Nevertheless, diffusion alone often falls short when robots need strict stability and safety guarantees--constraints more amenable to model-based optimization modules, such as trajectory optimizers, feedback controllers, and safety filters.
Integrating model-free diffusion with model-based optimization is challenging, because these modules often struggle with alignment.
Here, \textit{alignment} for modules sharing an output space (e.g., robot trajectories) means minimizing the difference between their outputs.
Alignment is particularly difficult for optimization-based safety-filters~\cite{tearle2021predictive, hsu2023safety, wabersich2023data, jung2024rail}, where diffusion to complete tasks can produce unsafe trajectories, while enforcing safety can degrade performance by negating generative diversity.

Alignment typically follows two strategies. 
The \textit{sequential} approach minimally corrects diffusion outputs via model-based postprocessing \cite{tearle2021predictive, hsu2023safety, jung2024rail, wabersich2021predictive, power2023sampling}, often causing module contradiction or impedance. 
Alternatively, \textit{embedding} incorporates constraints into diffusion denoising \cite{carvalho2024motion, mizuta2024cobl, kondo2024cgd, romer2024diffusion, xiao2023safediffuser}, but can sacrifice hard guarantees or apply only to simple cases. 
Both strategies suffer from being \textit{one-directional} (lacking reciprocal information sharing) and relying solely on diffusion for task objectives, such as task completion. 
Thus, we address the key challenge of how to simultaneously align model-free diffusion, model-based optimization, \textit{and} an external objective, achieving what we call \textit{mutual compatibility}.

To address the mutual compatibility challenge, our key insight is to first pose it as a \textit{Joint Model-Free Model-Based Generation} (JM2G) problem of explicitly sampling simultaneously from model-free and model-based parameter spaces.
To solve the JM2G problem, we take advantage of diffusion models as a generative framework and develop a \textit{Joint Model-based Model-free Diffusion} (JM2D).
To enable JM2D, we construct a joint distribution over the diffusion prior and the optimization prior, each importance weighted by an \textit{interaction potential} that quantifies mutual compatibility, even in the absence of paired training data.
As shown in \Cref{fig:teaser}, JM2D enables simultaneously diffusing over model-based and model-free modules while aligning both with each other and an external objective.
Altogether, JM2D resolves the core mutual compatibility challenge by yielding trajectories that satisfy arbitrary, possibly non‐differentiable, safety constraints by construction while preserving generative diversity.

\textbf{Contributions and Paper Organization.}
We pose the JM2G problem to enable mutual compatibility between model-free diffusion modules, model-based optimization modules, and external objectives, introducing a novel interaction potential to relate these components (\Cref{sec:problem}). To solve this, we propose the JM2D framework (\Cref{sec:method}). JM2D is based on a theoretical joint score function for simultaneous diffusion; critically, we derive a Monte Carlo score approximation that enables practical implementation even when model-based gradients are unavailable (a common scenario) and empirically improves diffusion output quality. Finally, we provide extensive experimental evaluations on offline RL benchmarks and robotic manipulation tasks (\Cref{sec:experiments}), demonstrating that, compared to SOTA methods with sequential sampling alignment, our joint sampling method achieves significantly improved task performance without sacrificing the safety constraint.

\begin{figure}
    \centering
    \includegraphics[width=0.80\textwidth]{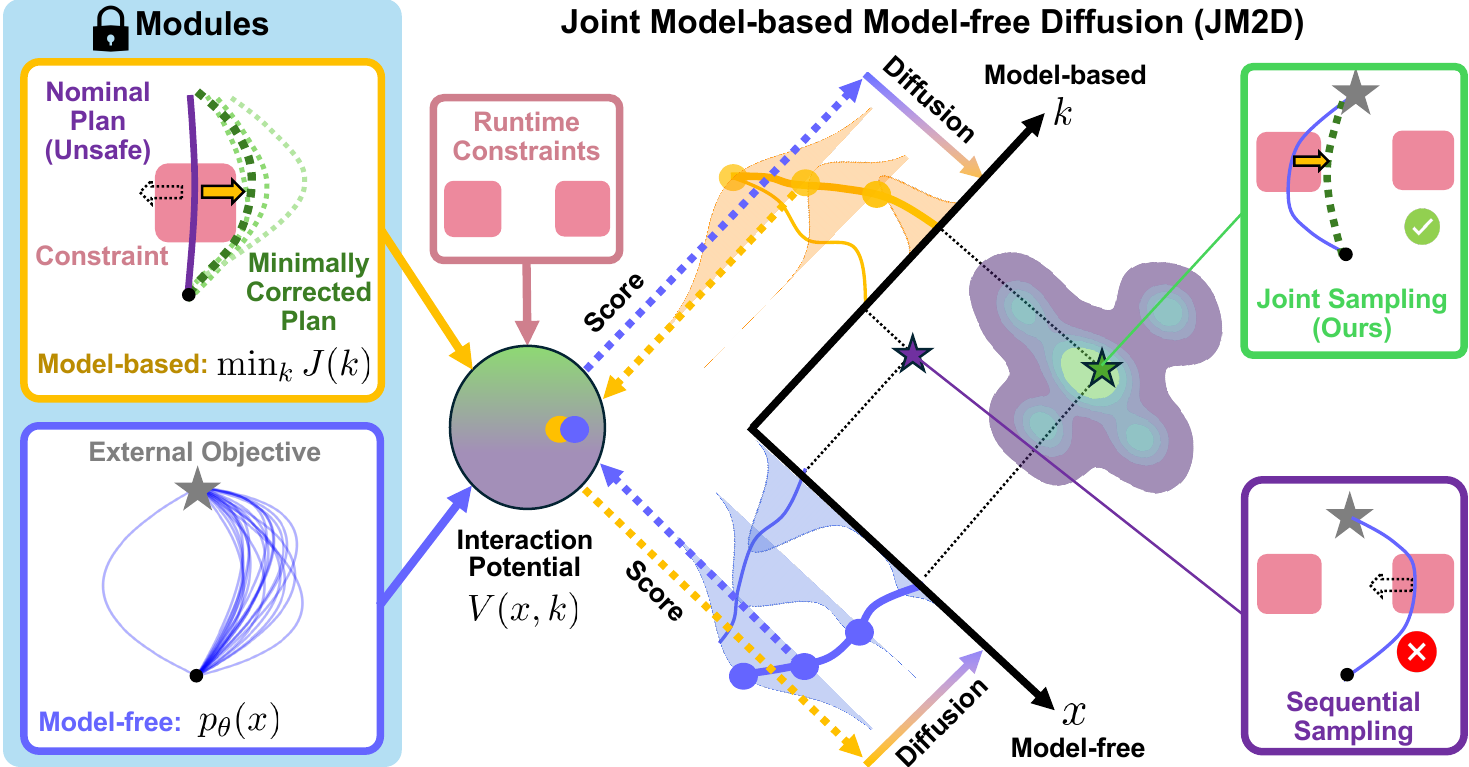}
    \caption{
    \textbf{Framework Overview.}
    We consider robotic systems composed of a model-free diffusion planner and a model-based optimization module. In this example, the diffusion policy is trained on right-skewed goal-reaching behaviors, while the model-based module can only correct plans to the right. The key challenge is to jointly sample a plan and its correction such that both modules are aware of each other’s capabilities. We propose Joint Model-based Model-free Diffusion (JM2D), which defines an interaction potential combining the diffusion model, optimization module, and constraints. Starting from joint noise, our method denoises both components together, yielding mutually compatible samples (\uncertified{green} box). In contrast, conventional sequential sampling fails to account for the optimization module's limitations, resulting in misaligned outputs (\purple{purple} box).
    } \label{fig:teaser}
    \vspace{-1em}
\end{figure}

\vspace{-0.5em}
\section{Related Work}
\vspace{-0.5em}


\textbf{Diffusion models and motion planning.} 
Motion planning seeks trajectories that drive a robot from a start to a goal state while respecting kinodynamic and environmental constraints. Traditional planners—such as gradient-based optimizers~\cite{ratliff2009chomp} and sampling-based methods~\cite{williams2018information}—perform well in structured settings but falter in highly nonconvex spaces or under discontinuous dynamics. Recently, diffusion models~\cite{sohl2015deep, song2020score, ho2020denoising}, which learn to sample from complex, multimodal distributions, have been applied to 
both data-driven~\cite{chi2023diffusion, ajay2022conditional, janner2022planning, chen2024simple} and model-based~\cite{pan2024model, xue2024full} motion planning. However, the stochastic sampling process complicates the satisfaction of explicit safety and feasibility constraints, limiting its deployment in safety-critical robotic domains.



\textbf{Diffusion planning with inference-time constraints.} Existing diffusion-based planners enforce constraints through three main paradigms.
(i) \textit{Post-hoc filtering} first generates diverse trajectories, then applies constrained optimization~\cite{power2023sampling, li2024constraint} or safety filters~\cite{hsu2023safety, wabersich2023data}—drawing on set‐invariance theory~\cite{ames2019control}, reachability analysis~\cite{bansal2017hamilton, jung2024rail}, or predictive control~\cite{tearle2021predictive}—but decoupled overrides can induce out‐of‐distribution actions~\cite{reichlin2022back} and efficiency losses~\cite{jung2024rail}.
(ii)\textit{ Soft-constraint guidance} uses classifier-based or classifier‐free methods~\cite{dhariwal2021diffusion, botteghi2023trajectory} and training‐free gradient steering~\cite{ye2024tfg, carvalho2024motion, mizuta2024cobl, kondo2024cgd}, yet demands differentiable losses and lacks strict feasibility guarantees.
(iii) \textit{Hard-constraint guidance} embeds convex feasibility via mirror maps or reflected SDEs~\cite{liu2023mirror, lou2023reflected, fishman2023metropolis} or projection‐based corrections~\cite{romer2024diffusion, bouvier2025ddat, liang2025simultaneous, christopher2025constrained}, but is limited to convex domains and can be computationally prohibitive. 
Other efforts such as~\cite{giannone2023aligning} integrate optimization during training under same-modality assumptions.
Our approach generalizes beyond this by jointly sampling across modalities—plan and optimization output—via a compatibility-aware diffusion process that satisfies hard constraints without performance loss, addressing prior limitations.

\vspace{-0.5em}
\section{Problem Statement and Preliminaries} 
\label{sec:problem}
\vspace{-0.5em}

This section introduces the general framework of \emph{Joint Model-free Model-based Generation (JM2G)}, where we consider the problem of sampling from a joint distribution over a model-based and a model-free variable, given only their marginals and an interaction potential quantifying their compatibility.
This formulation naturally arises in robotics, where subsystems such as planners, controllers, and safety filters are developed independently and only integrated at deployment: diffusion planners generate actions~\cite{chi2023diffusion, janner2022planning, ajay2022conditional}, and model-based optimization modules enforce constraints~\cite{bansal2017hamilton, kousik2020bridging, wabersich2021predictive}. This creates a joint sampling problem requiring simultaneous generation of actions and corresponding responses.
We first define the general version of the JM2G problem:

\begin{problem}\label{prob:joint_generation}
    \textbf{(Joint Model-free Model-based Generation)}
    Given a model-free generative distribution $\prob_\theta(\sample)$ over sample $\sample \in \samplespace \subset \R^{n_{\sample}}$, a model-based prior $\prob(\param)$ over sample $\param \in \paramspace \subset \R^{n_{\param}}$ and an interaction potential $\interactionfunc(\sample,\param): \samplespace \times \paramspace: \rightarrow \R_{\geq0}$ quantifying the compatibility of $(x,k)$ pairs,  we aim to efficiently sample from the joint distribution:
        $\prob(\sample, \param) \propto \prob_{\theta}(\sample)\prob(\param)\interactionfunc(\sample, \param).$
\end{problem}

While we learn the parameterized distribution $\prob_{\theta}(\sample)$ using a diffusion model, we assume $\prob(\param)$ and $\interactionfunc(\sample, \param)$ are only available at inference time. This reflects modular deployment scenarios where generative planners are pre-trained, and model-based optimization modules are introduced later.

In the context of robotics, we focus on a prevalent setting where the model-free component is a diffusion planner, and the model-based component is defined through \emph{constrained optimization} 
\footnote{
    Our framework naturally includes the setting where the model-based module is defined via an explicit form $\optvar = f(\sample)$ rather than an implicit form of optimization.
    We omit this extension for simplicity.
}.
Specifically:
(a) we use a diffusion planner to learn the distribution of plan $\sample \equiv [\sample_{t}, \dots, \sample_{t+H}]$ conditioned on the observation $\observation$, and
(b) an optimization module that generates a model-based output $\optvar$ (e.g., low-level control, safety correction) by solving
\begin{equation} \label{eq:optimization}
    \param^{*} = \argmin_{\param} \; \objective(\param \mid \sample) \quad \text{s.t.} \quad \ineqconstraint(\param \mid \sample) \leq 0
\end{equation}
\noindent where $\objective: \samplespace \times \paramspace \rightarrow \R$, $\ineqconstraint: \samplespace \times \paramspace \rightarrow \R$ respectively denotes the objective and constraints.
We aim to solve \Cref{prob:joint_generation} with the interaction potential defined as
    $V(\sample, \param) =
        \exp\left(-\lambda\inv\objective(\param \mid \sample)\right) \cdot 
            \indicator\left(\ineqconstraint(\param \mid \sample) \leq 0\right)$,
where the temperature $\lambda \in \R^{+}$ controls the optimality and $\indicator(\cdot)$ is an indicator function. We provide a toy domain planning example in \Cref{fig:donut2d_joint}.

\noindent

\textbf{Diffusion Models.}
We consider diffusion models~\cite{song2020score, ho2020denoising, song2020denoising} as a generative framework. 
Given a data distribution $\prob(\sample_0)$, the forward noising process perturbs clean samples $\sample_0$ into noisy samples $\sample_i$: $
\prob_{\alpha_i}(\sample_i \mid \sample_0) = \mathcal{N}(\sample_i; \sqrt{\alpha_i} \sample_0, (1-\alpha_i) I),
$ following a noise schedule $\alpha_i \in (0,1]$.
A reverse process denoises $\sample_i$ to $\sample_{i-1}$ using DDIM \cite{song2020denoising}: 
\begin{equation} \small \label{eq:reverse_denoise}
\prob_{\theta}(\sample_{i-1} \mid \sample_i) = \mathcal{N}\left(
\sample_{i-1}; 
\sqrt{\alpha_{i-1}} \left( \frac{\sample_i + (1-\alpha_i) s_\theta(\sample_i, i)}{\sqrt{\alpha_i}} \right) - \sqrt{1-\alpha_{i-1}-\sigma_i^2} \sqrt{1-\alpha_i} s_\theta(\sample_i, i),
\sigma_i^2 I
\right),
\end{equation}
where $s_\theta(\sample_i, i)$ predicts the \emph{score function} $\nabla_{\sample_i} \log \prob_i(\sample_i)$, the gradient of the log-density of noisy samples. %
The model is trained by minimizing the denoising score matching loss  \cite{song2020score} by regressing over the Tweedie score of the forward kernel \cite{luo2022understanding}:
$
s_{\alpha_i}(\sample_0 \mid \sample_i) = -\frac{\sample_i - \sqrt{\alpha_i} \sample_0}{1-\alpha_i}.$

In this work, we extend diffusion models to the \emph{joint space} $(\sample, \optvar)$, aiming to approximate the joint score $\nabla_{\sample, \optvar} \log p(\sample, \optvar)$ when only the model-free score $\nabla_{\sample} \log \prob_{\theta}(\sample)$ is available.

\vspace{-0.5em}
\section{Method}
\label{sec:method}
\vspace{-0.5em}

Na\"ively chaining diffusion planners and model-based optimizers often fails due to conflicting objectives. Our key insight is a unified view: simultaneously generating mutually-compatible outputs via JM2D (Joint Model-based Model-free Diffusion). This section develops JM2D, starting with a unified joint diffusion process encoding mutual compatibility via a shared score function (\Cref{sec:joint_vs_sequential}).
Since direct score computation is often intractable---due to non-differentiable objectives and reliance on clean samples---we derive a practical Monte Carlo approximation (\Cref{sec: monte carlo score estimator}).
This allows handling hard, non-differentiable constraints without privileged information such as gradient or compromising the data fidelity of the pre-trained planner.
Finally, we connect JM2D to conditional generation \cite{huang2024symbolic} and model-based diffusion \cite{pan2024model} as special cases (\Cref{sec: special cases conditional gen and MBD}).

\subsection{Joint Diffusion Process}
\label{sec:joint_vs_sequential}

\textbf{Challenge.}
A natural joint generation approach is sequential sampling: first drawing $\sample \sim \prob_{\theta}(\sample)$ from the model-free planner, then solving for $\param \sim \prob(\param \mid \sample)$ via model-based optimization.
Alternatively, one could alternate these steps using Gibbs sampling~\cite{gilks1995markov}. However, in robotics, planner outputs and optimization are often tightly coupled, such as when a planner proposes an unsafe trajectory, leaving no feasible corrective action \cite{hsu2023safety, jung2024rail, wabersich2021predictive}.
As a result, many $\sample$ admit no compatible $\param$, making $\prob(\param \mid \sample)$ ill-defined and causing sample waste.
More formally, when $\prob(\param \mid \sample)$ is sparsely supported in $\samplespace \times \paramspace$, joint sampling becomes essential to efficiently explore feasible regions.

\begin{figure}[b]
    \centering
    \includegraphics[width=0.80\linewidth]{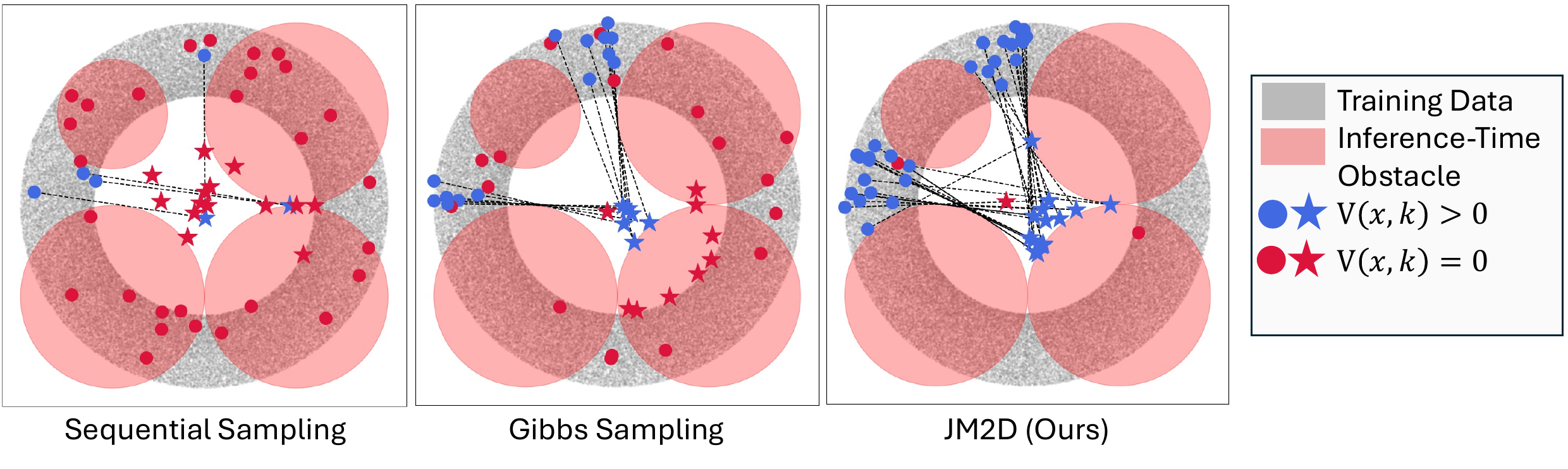}
    \caption{
    \textbf{Comparison of sampling methods on a toy domain.}
    A model-free planner $\prob_{\theta}(\sample)$ samples start-goal pairs in a donut-shaped region (gray), and a model-based optimization module finds the longest collision-free path connecting start and goal parameterized by $\optvar$ while avoiding inference-time obstacles (transparent red).
    We visualize 16 samples per sampling method. 
    Sequential sampling draws $\sample$ without considering $\optvar$, often yielding infeasible pairs (red).
    Gibbs sampling alternates between $\prob(\sample \mid \param)$ and $\prob(\param \mid \sample)$, improving feasibility but struggling when conditionals are ill-defined.
    Instead, JM2D (Ours) jointly explores $(x, k)$ via diffusion, globally searching the compatible joint sample, achieving the highest rates of mutually compatible samples (blue).
    }
    \label{fig:donut2d_joint}
    \vspace{-1em}
\end{figure}

\textbf{Design.} To address the challenge, we introduce a joint diffusion process over the concatenated variable $\jointsample =[\sample, \optvar]$. 
Specifically, we define a pseudo-forward diffusion:
\begin{equation} \small
\label{eq:joint_forward_noise}
    \prob_{\alphajoint_i}(\jointsample_i \mid \jointsample_0) = \mathcal{N}\left(\jointsample_i \mid \sqrt{\smash[b]{\alphajoint_i}} \jointsample_0, (1-\alphajoint_i) I\right),
\end{equation}
where $\alphajoint_i = [\alphasample_i; \alphaoptvar_i]$ preserves the pre-trained model-free schedule $\alphasample_i$ while letting $\alphaoptvar_i$ explore the optimization variable $\optvar$.
For sampling, starting from pure noise, we iteratively apply the reverse update \eqref{eq:reverse_denoise} to recover a compatible joint sample.
To perform this reverse process, we require the \emph{joint score}
$\nabla_{\jointsample_i} \log \prob_i(\jointsample_i)$, where 
$\prob_i(\jointsample_i) = \int \prob_{\alphajoint_i}(\jointsample_i \mid \jointsample_0) \prob_0(\jointsample_0) d\jointsample_0$.
This score guides $\jointsample_{i}$ toward higher compatibility under the interaction potential without excessively perturbing the pre-trained model-free score.
Hence, our joint diffusion framework can \emph{globally search} over $\jointsample$ while \emph{locally refining} according to mutual compatibility.
\Cref{fig:donut2d_joint} shows how joint diffusion outperforms sequential sampling in capturing compatible joint solutions.
\subsection{Monte-Carlo Joint Score Approximation} \label{sec: monte carlo score estimator}
\textbf{Challenge.} 
We now efficiently estimate the joint score $\nabla_{\jointsample_{i}} \log\prob_{i}(\jointsample_{i})$ to implement the joint diffusion scheme from \Cref{sec:joint_vs_sequential}.
This approximation is challenging in robotics for three reasons:
\emph{(a) Non-differentiability}: the interaction potentials often encode hard constraints or nonconvex optimizations, so gradients are unavailable.
\emph{(b) Clean sample dependence}: compatibility can be only evaluated on clean (fully denoised) samples, preventing naive noisy levels score composition.
\emph{(c) Data-fidelity}: the planner is pre-trained; overly aggressive corrections can destroy its learned distribution \cite{wang2024inference}, so we must minimally perturb the model-free score $\nabla_{\sample_i}\log\prob(\sample_i)$.

\textbf{Design.} 
To address these, we derive a gradient-free Monte Carlo (MC) estimator that approximates the joint score via importance sampling, which requires only \emph{evaluation of the interaction potential at clean joint samples} to guide the generation along the model-free reverse diffusion process. 

We derive our estimator starting from
\begin{align} \small\label{eq:joint_score_general}
    \nabla_{\jointsample_{i}}\log\prob_{i}(\jointsample_{i})
    &= \frac{
        \int \nabla_{\jointsample_{i}} \log\prob_{\alphajoint_i}(\jointsample_{i} | \jointsample_{0})
        \prob_{\alphajoint_i}(\jointsample_{i} | \jointsample_{0})\prob_{0}(\jointsample_{0})
        d\jointsample_{0}
    }{
        \prob_{i}(\jointsample_{i})
    } 
    = \int s_{\alpha^{\jointsample}_{i}}(\jointsample_{0}|\jointsample_{i}) \prob_{0|i}(\jointsample_{0}|\jointsample_{i}) d\jointsample_{0}
\end{align} 
\noindent where $s_{\alpha^{\jointsample}_{i}}(\jointsample_{0}|\jointsample_{i})$ is the Tweedie score of the forward kernel \cite{luo2022understanding}.
See \Cref{app:extended_proof:score_lemma} for details.
Since the posterior $\prob_{0|i}(\jointsample_0 \mid \jointsample_i)$ is intractable, we instead approximate it using self-normalized importance sampling with a tractable proposal $q(\jointsample_0 \mid \jointsample_i)$, giving:
\begin{align} \small
\label{eq:joint_score_general_proposal}
    \nabla_{\jointsample_{i}}\log\prob_{i}(\jointsample_{i}) 
    = \frac{
        \int s_{\alpha^{\jointsample}_{i}}(\jointsample_{0}|\jointsample_{i}) \frac{\prob_{0|i}(\jointsample_{0}|\jointsample_{i})}{q(\jointsample_{0}|\jointsample_{i})} q(\jointsample_{0}|\jointsample_{i}) d\jointsample_{0}
    }{
        \int \frac{\prob_{0|i}(\jointsample_{0}|\jointsample_{i})}{q(\jointsample_{0}|\jointsample_{i})} q(\jointsample_{0}|\jointsample_{i})  d\jointsample_{0}
    } \approx \frac{ 
        \mathbb{E}_{\hat{\jointsample}_0 \sim q(\cdot \mid \jointsample_{i})}\left[s_{\alphajoint_{i}}(\hat{\jointsample}_0 \mid \jointsample_i) w(\hat{\jointsample}_0)\right]
    }{ 
        \mathbb{E}_{\hat{\jointsample}_0 \sim q(\cdot \mid \jointsample_{i})}\left[w(\hat{\jointsample}_0)\right]
    }
\end{align}
\noindent where $w(\hat{\jointsample}_0)$ are the importance weights.
We define the proposal as 
\begin{equation} \small\label{eq:proposal}
    q(\jointsample_0 \mid \jointsample_i) = \prob_{0|i}^{\sample}(\sample_0 \mid \sample_i) \cdot q^{\optvar}(\optvar_0 \mid \optvar_i),
    \quad q^{\optvar}(\optvar_0 \mid \optvar_i) \triangleq \normal\left(\optvar_0; \optvar_i(\alphaoptvar_i)^{-\sfrac{1}{2}}, \left(\tfrac{1}{\alphaoptvar_i} - 1\right) \eye \right),
\end{equation}
where $\prob^{\sample}_{0|i}(\sample_0 \mid \sample_i)$ is the model-free reverse diffusion. This leads to our main theorem:
\begin{thm} \label{thm:main}
\textbf{Joint Score Approximation using Importance Sampling.} The joint score can be expressed as an importance-weighted combination of the scores of individually estimated clean joint samples, where the weights are determined by the interaction potential:
\begin{equation} \small\label{eq:score_approximation_joint}
    \nabla_{\jointsample_{i}} \log\prob_{i}(\jointsample_{i}) = 
    \frac{
            \expectation_{\hat{\jointsample}_{0} \sim q(\cdot | \jointsample_{i})} 
            [s_{\alphajoint_{i}}(\hat{\jointsample}_0 \mid \jointsample_{i})
            \interactionfunc(\hat{\sample}_{0}, \hat{\optvar}_{0})\prob^{\optvar}_{0}(\hat{\optvar}_{0})]
        }{
            \expectation_{\hat{\jointsample}_{0} \sim q(\cdot | \jointsample_{i})}
            [\interactionfunc(\hat{\sample}_{0}, \hat{\optvar}_{0})\prob^{\optvar}_{0}(\hat{\optvar}_0)]
        }
\end{equation}
\noindent where: $\jointsample_{i}=[\sample_{i}; \optvar_{i}]$ denotes the joint noisy sample, $q(\cdot \mid \jointsample_{i})$ denotes \eqref{eq:joint_score_general_proposal}.
\begin{proof}
This follows from $w(\hat{\jointsample}_0) = \frac{\prob_{0|i}(\hat{\jointsample}_{0}|\jointsample_{i})}{q(\hat{\jointsample}_{0}|\jointsample_{i})} \propto \interactionfunc(\hat{\sample}_{0}, \hat{\optvar}_{0})\prob^{\optvar}_{0}(\hat{\optvar}_0)$.
See \Cref{app:extended_proof_main_thm} for the derivation.
\end{proof}
\end{thm}
Having defined the proposal $q(\jointsample_0 \mid \jointsample_i)$, we estimate the joint score using MC samples and apply it in a reverse diffusion \eqref{eq:reverse_denoise}. 
See \Cref{algo:joint_sis} for details.
For extensions to settings with privileged information as gradients and projection operators, see \Cref{app:discussion}.

\begin{algorithm*}[ht]
\caption{Joint Model-free Model-based Diffusion (JM2D)}
\label{algo:joint_sis}
\begin{algorithmic}[1]
\Require Model-free diffusion $\prob_{\theta}(\sample)$, Model-based prior $\prob^{\optvar}_{0}(\optvar_0)$, Interaction potential $\interactionfunc(\sample, \optvar)$
\Require Joint forward noise schedule $\alpha^{\jointsample}_{i} = [\alpha^{\sample}_{i}, \alpha^{\optvar}_{i}]$, Stochasticity $\sigma_i$
\State Initialize joint noise: $\jointsample_{I} = [\sample_{I}; \optvar_{I}] \sim \normal(\zeros, \eye)$
\For{$i = I$ to $1$}
    \State Construct the Monte Carlo samples per  \eqref{eq:proposal}:
    \[
    \samplespace_{i}, \paramspace_{i} \leftarrow \textsc{SampleMC} (i, \sample_i, \optvar_i, \prob_{\theta}(\sample), \alpha_i^{\jointsample})
    \]
    \State Estimate the joint score per \eqref{eq:score_approximation_joint}:
    \[
    \nabla_{\jointsample_{i}} \log\prob_{i}(\jointsample_{i})
    =
    \colvec{
        \nabla_{\sample_{i}} \log\prob_{i}(\jointsample_{i})
    }{
        \nabla_{\optvar_{i}} \log\prob_{i}(\jointsample_{i})
    }
    \leftarrow
    \frac{
        \sum_{(\hat{\sample}_0, \hat{\optvar}_0) \in \samplespace_i \times \paramspace_i}
        \left[s_{\alphajoint_i}(\hat{\jointsample}_0 \mid \jointsample_i)
        \interactionfunc(\hat{\sample}_0, \hat{\optvar}_0)\prob^{\optvar}_{0}(\hat{\optvar}_0)\right]
    }{
        \sum_{(\hat{\sample}_0, \hat{\optvar}_0) \in \samplespace_i \times \paramspace_i}
        \left[ \interactionfunc(\hat{\sample}_0, \hat{\optvar}_0)\prob^{\optvar}_{0}(\hat{\optvar}_0) \right]
    }
    \]
    
    \State Denoise the joint samples per \eqref{eq:reverse_denoise} and Eq.~(6) in \cite{pan2024model}:
    \Statex
    \begin{align*}
    \jointsample_{i-1} \leftarrow\ &
    \sqrt{ \frac{\alpha_{i-1}^{\jointsample}}{\alpha_{i}^{\jointsample}} }
    \left(
    \jointsample_i + (1 - \alpha_i^{\jointsample}) \nabla_{\jointsample_{i}} \log \prob_{i}(\jointsample_{i})
    \right) \\
    &+
    \begin{bmatrix}
    \sqrt{1 - \alpha_{i-1}^{x} - \sigma_i^2} \sqrt{1 - \alpha_i^{x}} \nabla_{\sample_{i}} \log \prob_{i}(\jointsample_{i}) + \sigma_i z \\
    0
    \end{bmatrix},
    \quad z \sim \normal(0, \eye)
\end{align*}
\EndFor
\State \Return $\jointsample_0 = [\sample_0; \optvar_0]$
\end{algorithmic}
\end{algorithm*}

\begin{algorithm*}[h]
\caption{Monte Carlo Sample Construction (\textsc{SampleMC})}
\label{algo:mc_sampling}
\begin{algorithmic}[1]
\Require Denoising timestep $i$, Noisy model-free sample $\sample_i$, Noisy model-based sample $\optvar_i$
\Require Diffusion model $\prob_{\theta}$, Joint forward noise schedule $\alpha_{i}^{\jointsample}$
\State Sample $N$ clean model-free samples $\hat{\sample}_{0}^{(j)}$ from $\sample_i$ via reverse steps:
\For{$\tau = i$ to $1$}
    \State $\hat{\sample}_{\tau-1}^{(j)} \sim \prob_{\theta}(\sample_{\tau-1}^{(j)} \mid \hat{\sample}_{\tau}^{(j)})$, for $j=1 \cdots N$
\EndFor
\State Sample $N_K$ clean model-based parameters:
\[
\hat{\optvar}_0^{(l)} \sim q^{\optvar}(\optvar_0 \mid \optvar_i) = 
\normal\left(
    \optvar_0; \optvar_i (\alpha_i^{\optvar})^{-\sfrac{1}{2}}, \left(\tfrac{1}{\alpha_i^{\optvar}} - 1\right) \eye
\right), \quad l = 1 \cdots N_K
\]
\State \Return $\{ \hat{\sample}_{0}^{(j)} \}_{j=1}^{N}$, $\{ \hat{\optvar}_0^{(l)} \}_{l=1}^{N_K}$
\end{algorithmic}
\end{algorithm*}

\textbf{Remark.}
In the robotics context, by encoding mutual compatibility, the interaction potential encourages sampling feasible plans, but does not guarantee feasibility ($\ineqconstraint(\param \mid \sample) \leq 0$ per \eqref{eq:optimization}).
Therefore, in practice, we postprocess infeasible joint samples (i.e., $V(\sample_0, \optvar_0) = 0$) with a single model-based optimization step to guarantee feasibility.
Despite this, our approach still improves performance because the sample is already aligned with the optimization (see \Cref{sec:experiments}).

\subsection{Special Cases: Conditional Generation and Model-Based Diffusion}\label{sec: special cases conditional gen and MBD}

We conclude this section by showing that JM2D unifies Conditional Generation (CG) and Model-Based Diffusion (MBD) \cite{pan2024model} under a common framework.
This is crucial because these approaches represent widely-used paradigms in diffusion-based robotics, so unification supports direct benchmarking with CG-based methods and enables the reuse of practical heuristics from the MBD literature, such as proposal shaping and noise scheduling.

\begin{cor} \label{cor:special_case}
Under the interaction potential structure $\interactionfunc(\sample, \optvar) = \interactionfunc^{\sample}(\sample) \interactionfunc^{\optvar}(\optvar)$ and a uniform prior $\prob(\optvar)$, JM2D reduces to two independent sampling process:  
(1) conditional generation of $\sample$ guided by $\interactionfunc^{\sample}(\sample)$, and  
(2) model-based diffusion over $\optvar$ guided by $\interactionfunc^{\optvar}(\optvar)$.
\end{cor}

\begin{proof}
\vspace{-1em}
Follows directly from \Cref{thm:main}.
See \Cref{app:extended_proof_cor} for the proof.
\vspace{-1em}
\end{proof}
As joint sampling lacks baselines, we benchmark the CG segment against constraint-aligned planners to evaluate sampling efficiency in the absence of model-based modules (see \Cref{subsec: eval for cond gen in robotics}).
\vspace{-0.5em}
\section{Experiments: Joint Sampling}
\label{sec:experiments}
\vspace{-0.5em}

In this section, we assess the effectiveness of joint diffusion in handling complex, non-differentiable interaction potentials in a robotics safety-filtering task.

\subsection{JM2D for Robot Motion Planning and Safety Filtering}
\label{subsec: eval on joint sampling robot planning}

\textbf{Setup.}
We aim to understand key performance improvements with JM2D in a safety-filtering scenario. 
Thus, we hypothesize that JM2D improves both safety \textit{and} task completion over baselines.
We evaluate in 2D D4RL-PointMaze (\Cref{fig:d4rl_pointmaze}) environment, where the goal is to reach the target without colliding with the walls. The state and action spaces are $\state \in \mathbb{R}^4$, representing position and velocity, and $\action \in \mathbb{R}^2$, denoting 2D actuation forces.
We use a pre-trained Diffusion Policy~\cite{chi2023diffusion} as our model-free module and a reachability-based safety policy~\cite{kousik2020bridging,jung2024rail} as the model-based optimization module.

\begin{figure}[h]
    \centering
    \includegraphics[width=0.75\textwidth]{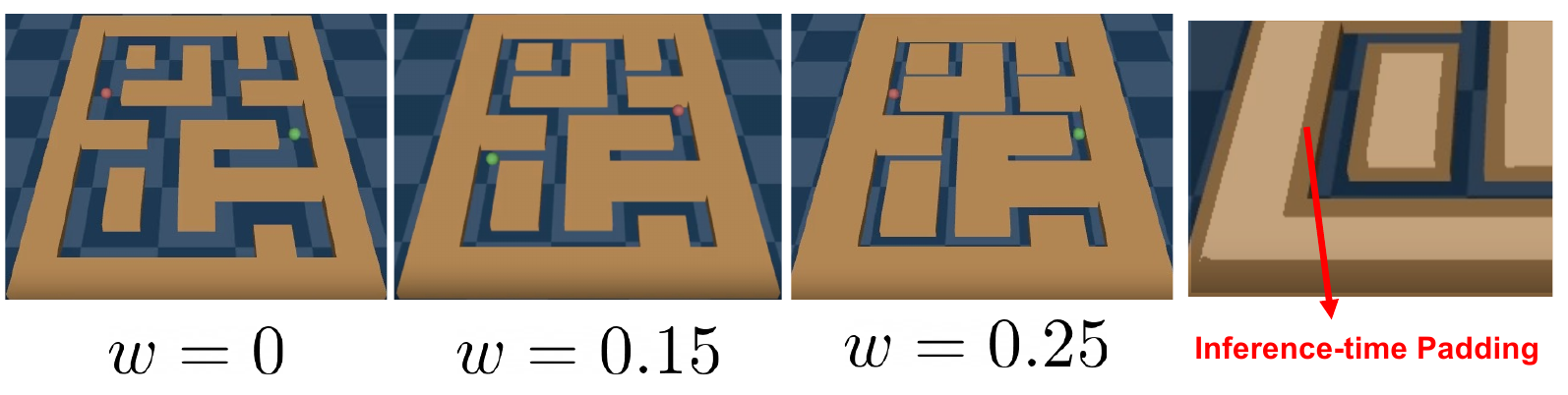}
    \caption{\textbf{Task description of D4RL-PointMaze \cite{fu2020d4rl}}.
    (a) The robot (red) must reach the goal (green) without colliding with the wall.
    (b) At inference time, the wall boundaries are inflated by a width $w$ to impose additional safety constraints. 
    The case $w = 0$ corresponds to the original wall thickness used during training.}
    \label{fig:d4rl_pointmaze}
\end{figure}

We use a pre-trained Diffusion Policy~\cite{chi2023diffusion} as our model-free module to generate sequences of actions $\sample = [\action_t, \dots, \action_{t+H}]$ over a fixed horizon of $H = 32$, resulting in $\sample \in \mathbb{R}^{64}$. As the safety policy, we use reachability-based safety policy (RTD)~\cite{kousik2020bridging,jung2024rail}, which computes a failsafe backup action $\optvar \in \mathbb{R}^2$ via a model-based optimization module. During evaluation, we introduce novel obstacles by inflating the wall boundaries, which were not present during training. Each policy is evaluated across 5 random seeds, with 100 simulations per seed.

We introduce inference-time constraints by padding the maze walls; the diffusion planner does not see padded walls during training.
We also perform an ablation to assess the tradeoff between computation time and accuracy by varying JM2D's number of Monte Carlo samples.
We compare to two baselines: (a) RAIL~\cite{jung2024rail}, which aligns a model-free diffusion policy~\cite{chi2023diffusion} with a reachability-based safety filter~\cite{kousik2020bridging} via sequential sampling; and (b) a Gibbs sampler that alternates between conditional distributions over model-free and model-based components. RAIL~\cite{jung2024rail} enforces the same safety filter as ours but performs filtering explicitly after sampling a diffusion plan. Such a method is highly susceptible to bad modes being sampled by the diffusion planner, resulting in no safe backup plans. Moreover, such methods lack feedback between the model-free planner and the model-based optimization module, which can result in conflicting actions and failure in completing the task objective.

We combine three metrics to assess mutual compatibility: \textbf{Safe Success Rate} (percentage of safe and successful trials), \textbf{Intervention Rate} (percentage of timesteps with safety filter intervention), and \textbf{Task Horizon} (duration to task completion).

\begin{figure}[h]
\vspace{-1em}
    \centering
    \includegraphics[width=\textwidth]{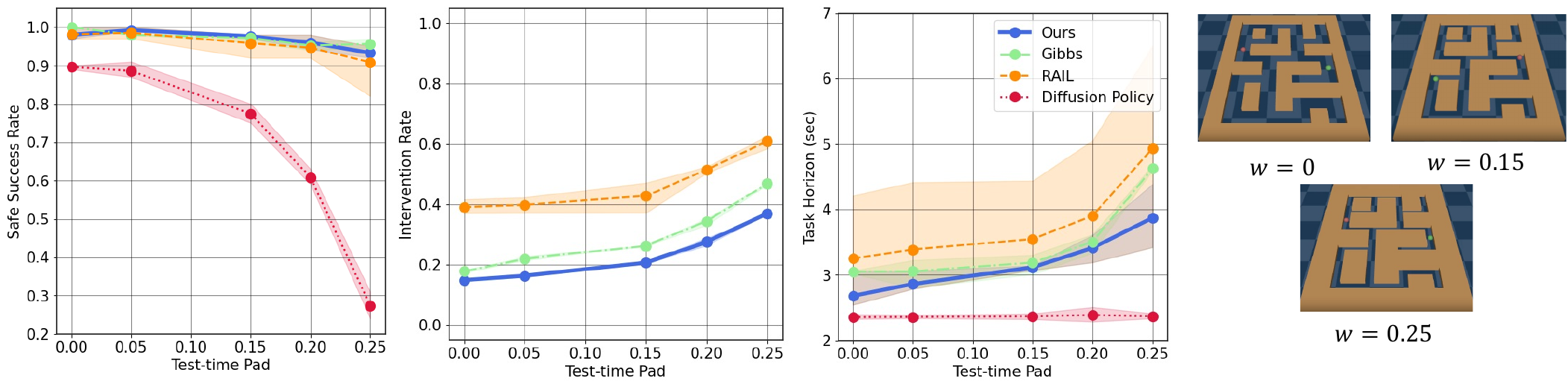}
    \caption{
    \textbf{Joint diffusion safety filtering increases performance without sacrificing safety.}
    A mobile robot (green dot) must navigate to a target (red dot) while avoiding collisions with walls.
    At test time, additional obstacles are introduced by padding walls with increasing thickness $w$.
    We evaluate performance across 80 trials of 5 random seeds using three metrics: \textit{safe success rate} (left), \textit{intervention rate} (middle), and \textit{task horizon} (right).
    As wall padding increases, without the safety intervention, the baseline diffusion policy's safe success rate degrades significantly. RAIL and Gibbs sampling strategy maintains safety but exhibits increasing intervention rates.
    In contrast, our method consistently exhibits high safe success, with significantly fewer interventions and shorter horizons—indicating improved safety-performance trade-offs.
    While RAIL, Gibbs and ours  all maintain strict safety guarantees, our joint sampling leads to more efficient task execution.
    }
    \label{fig:joint_maze_2d}
    \vspace{-1.5em}
\end{figure}

\textbf{Result: JM2D outperforms Sequential sampling and Gibbs sampling.} 
As shown in \Cref{fig:joint_maze_2d}, JM2D achieves significantly lower Intervention Rate and higher Safe Success Rate over RAIL, especially as conflicts between inference-time safety constraints and the training distribution increase (e.g., with greater wall padding, the safe success rate of the diffusion policy drops sharply). 
This improvement results from JM2D’s ability to jointly sample both the plan and its safety backup, maintaining compatibility throughout the diffusion process.
RAIL’s decoupled design lacks feedback, often producing incompatible samples that require external safety-filter overrides, leading to longer task horizons and reduced success.
While Gibbs sampling incorporates feedback, it still suffers from higher intervention rates and longer horizons, likely due to ill-posed intermediate steps that bias guidance and reduce compatibility.
None of these methods violate safety.


\begin{wrapfigure}{r}{0.4\linewidth}
    \centering
    \vspace{-15pt}
    \includegraphics[width=\linewidth]{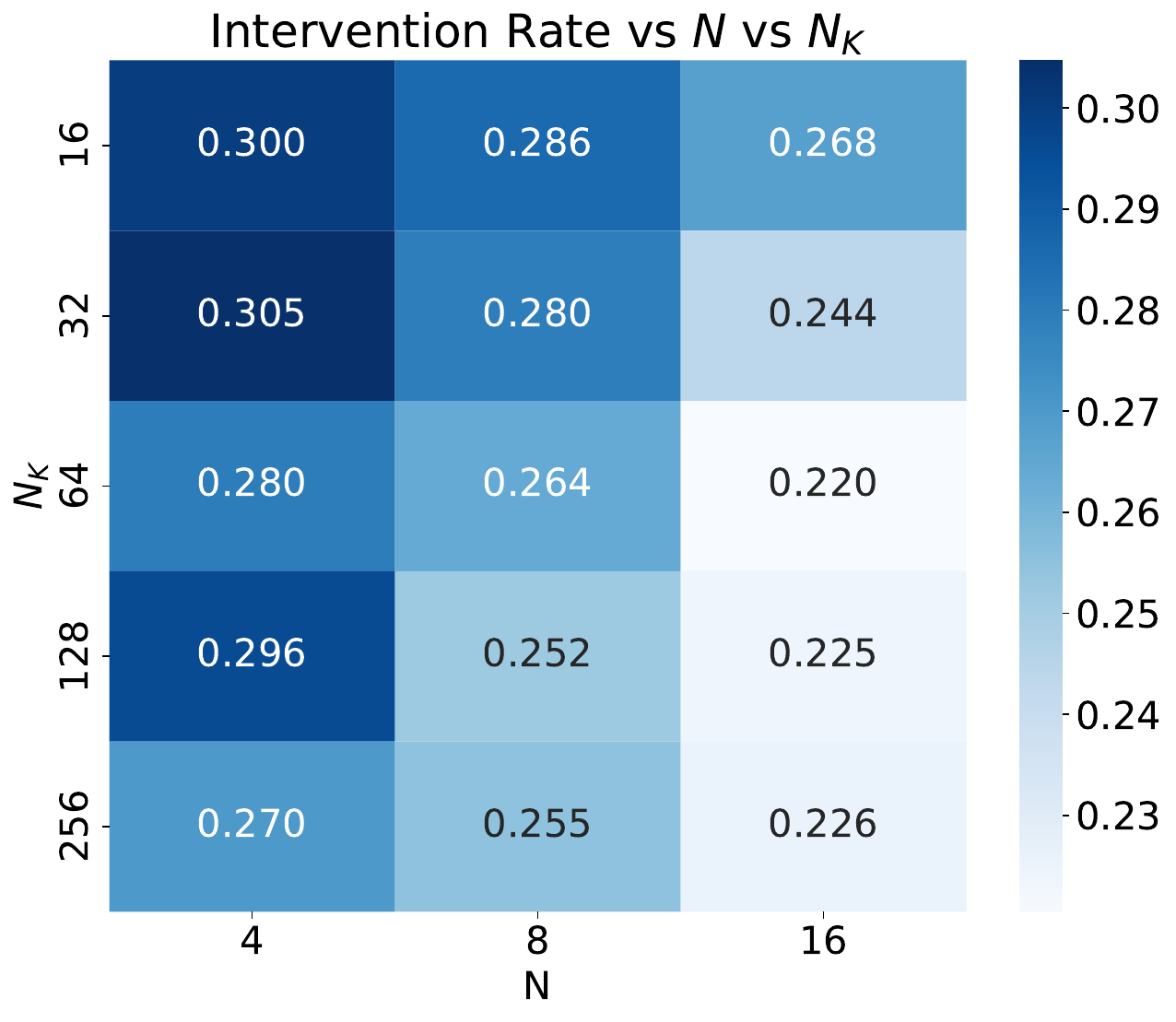}
    \vspace{-10pt}
    \caption{
    As the number of candidate samples ($N$ and $N_K$) increases, intervention of JM2D samples by the external safety filter decreases, implying higher CA.
    }
    \label{fig:jm2d_N_Nk}
    \vspace{-1em}
\end{wrapfigure}

\textbf{Result: Ablation on Monte Carlo Sample Size.} 
As shown in \Cref{fig:jm2d_N_Nk}, more samples reduce the Intervention Rate; this is due to exploring more denoising paths for $\sample$ and $\param$, enabling better mutual compatibility.
However, since this increases computation time, users must balance sample size with practical efficiency.

\textbf{Result: Ablation on Backup Planner Design Quality.}
Safety–filter frameworks must tolerate imperfect backup policies in real robots, so we evaluate \textsc{JM2D} under three backup policy variants:
(1) \emph{High quality}, which performs an accelerate--brake maneuver (accelerate for 0.5\,s, then decelerate to rest);
(2) \emph{x$+$y$+$}, which allows acceleration only in the positive \(x\) and \(y\) directions; and
(3) \emph{x$-$y$-$}, which allows acceleration only in the negative \(x\) and \(y\) directions.
For each variant we execute 50 roll‑outs with inference‑time wall padding \(w = 0.20\) and record the Safe Success Rate (SSucc.), Task Horizon (THor.), and Intervention Rate (Int\,\%), reported in Table~\ref{tab:app:robust}.
Constraining the backup planner sharply degrades \textsc{RAIL}:\;safe success falls from 0.90 to 0.60 while interventions climb from 0.52 to 0.72.
In contrast, \textsc{JM2D} remains resilient (SSucc.\,\(\ge 0.76\) across all variants) and roughly halves the intervention frequency, suggesting that its diffusion planner anticipates backup‑planner limitations and produces plans that remain compatible even with weakened safety capabilities.

\begin{table}[h] \small
    \centering
    \caption{Robustness of \textsc{JM2D} to backup–planner design qualities.}
    \begin{tabular}{|l|c|c|c|c|c|c|c|c|c|}
    \hline
       \multirow{2}{*}{Algorithm} & \multicolumn{3}{c|}{High quality} & \multicolumn{3}{c|}{Low quality ($x+y+$)} & \multicolumn{3}{c|}{Low quality ($x-y-$)}\\
       \cline{2-10} 
       & SSucc. & SHor. & Int\% 
       & SSucc. & SHor. & Int\%
       & SSucc. & SHor. & Int\%
       \\
       \hline 
       RAIL & 
       $0.90$ & $351.7$ & $0.52$ & 
       $0.62$ & $417.3$  & $0.61$ & 
       $0.60$ & $367.7$ & $0.72$ \\
       JM2D (Ours) & 
       $\textbf{0.94}$ & $\textbf{297.7}$ & $\textbf{0.24}$ &
       $\textbf{0.76}$ & $\textbf{259.7}$ & $\textbf{0.40}$ &
       $\textbf{0.88}$ & $\textbf{366.3}$ & $\textbf{0.47}$ \\
       \hline
    \end{tabular}
    \label{tab:app:robust}
    \vspace{-6pt}
\end{table}

\subsection{Real-world Evaluation}
\textbf{Setup} We implement JM2D on a Franka robot. For our task, we chose to pick up a mug while avoiding obstacles. These obstacles were not present in the training data for the diffusion policy. To train the model, we collect $100$ demonstrations of the task through teleoperation of the robot, grasping the mug in different locations like the rim and the handle. The input to the diffusion policy~(DP) is the SE(3) poses of the mug (acquired through ArUco markers \cite{garrido2014automatic}), and the proprioceptive information of the robot. The output of the model is the absolute end-effector pose (position and translation). During testing, we placed novel obstacles in the robot's workspace as shown in~\Cref{fig:real-robot}. The robot then has to complete the task without hitting any of the obstacles. 
\begin{figure}[h]
    \centering
    \includegraphics[width=0.5\linewidth]{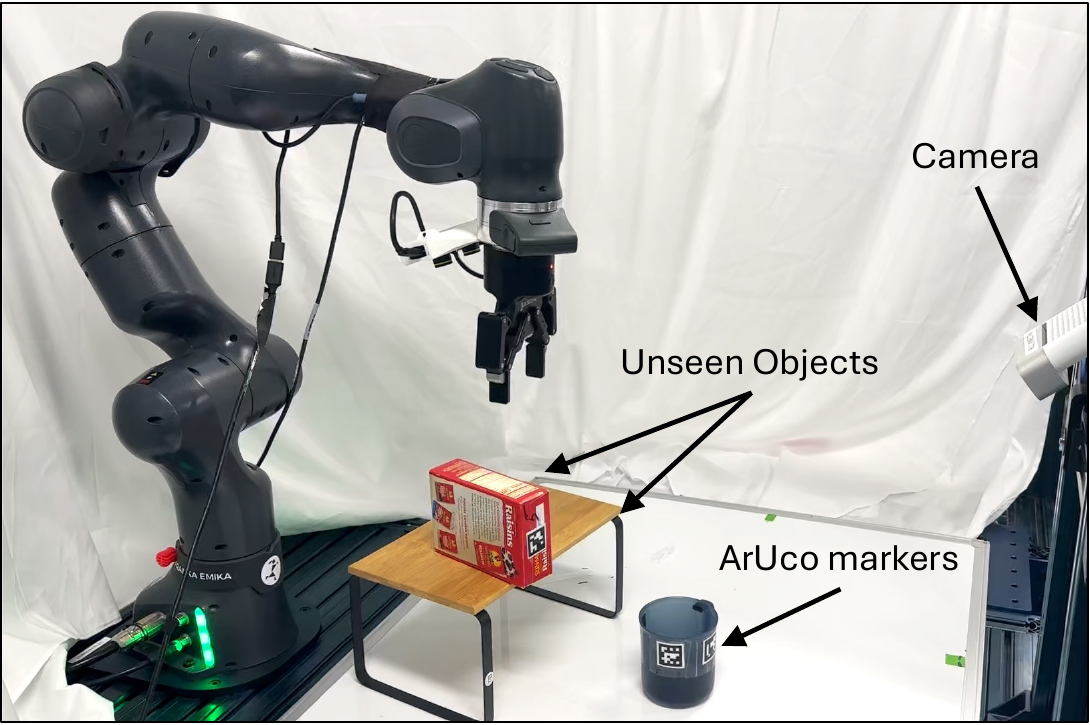}
    \caption{Real robot experiment setup}
    \label{fig:real-robot}
\end{figure}

We rollout the pre-trained vanilla DP in the cluttered scenario. Since there are obstacles in the scenario previously unseen by the DP, it collides with them and fails. This is expected as naive DP is uninformed action sampling. We show this in~\Cref{fig:real-rollout} (top). On the other hand, we deploy DP combined with a reachability-based safety filter deployed through our JM2D sampling framework. We see an improved safety guarantee as the sampled actions respond to the backup planner and follow a safe trajectory without compromising task success as illustrated in~\Cref{fig:real-rollout} (bottom).
\begin{figure}[h]
    \centering
    \includegraphics[width=\linewidth]{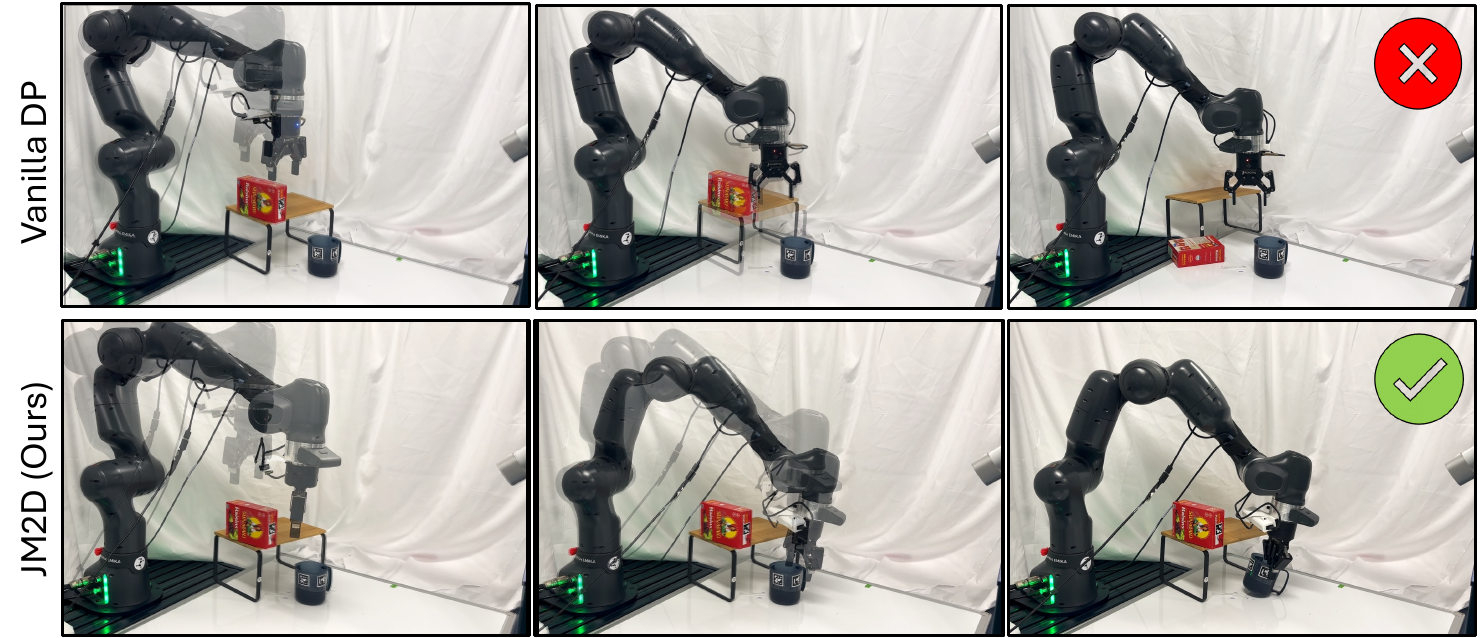}
    \caption{(Top) Na\"ive DP fails to execute the task of mug-pickup when it faces obstacles unseen in the training data. (Bottom) JM2D-based composition of DP and reachability-based safety filter successfully completes the task without compromising safety.}
    \label{fig:real-rollout}
\end{figure}

Additionally, investigating into the diffusion process also shows the multi-modality considered by our method before choosing the final trajectory. With our JM2D framework, we can guide the diffusion sampling process to choose the safest mode. We show it in~\Cref{fig:zono-rollout} and provide more implementation details in~\Cref{app:realworld}.

\begin{figure}[h]
    \centering
    \includegraphics[width=\linewidth]{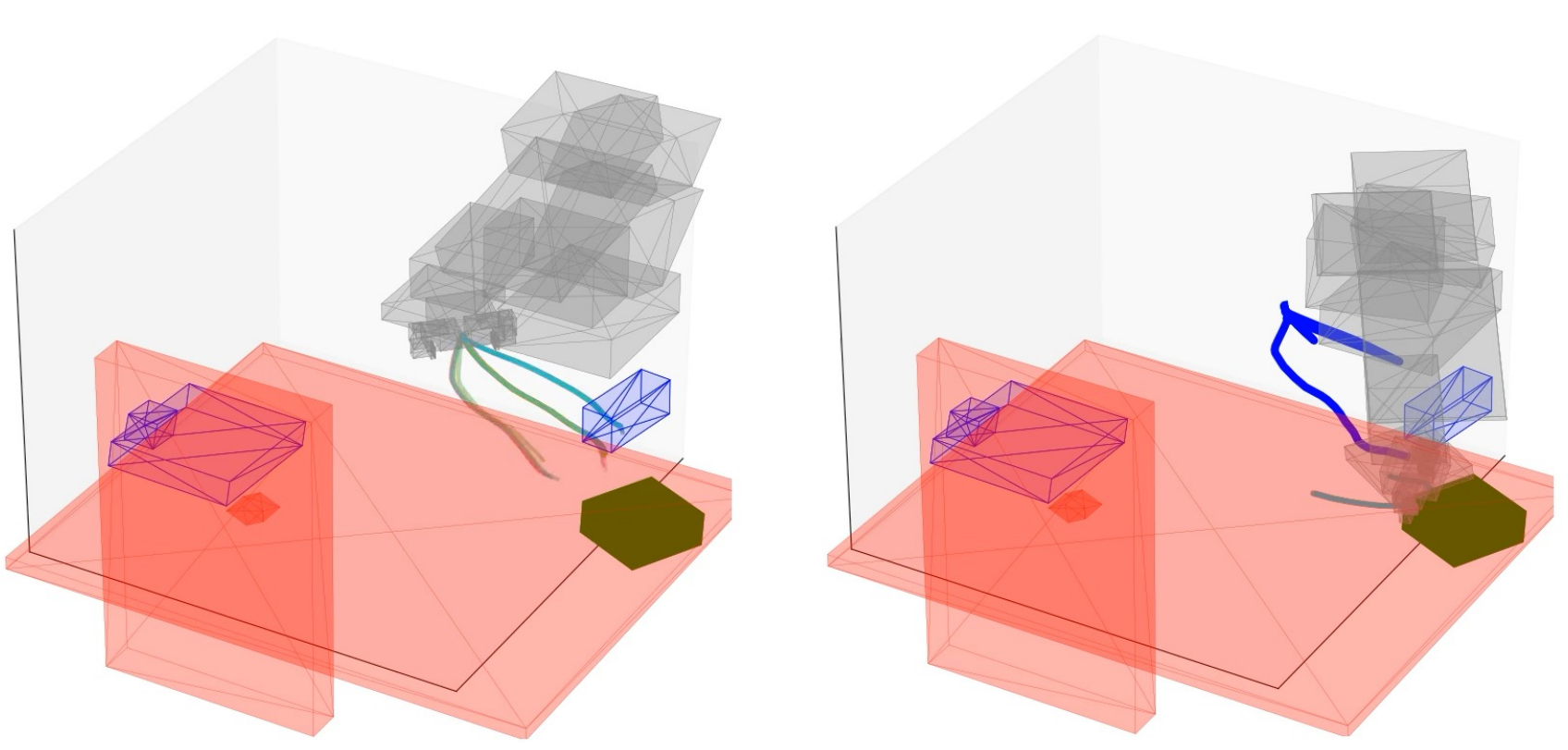}
    \caption{Visualization of planned trajectories and safety filter collision volumes; static obstacles are red, inference-time obstacles are blue, the robot is grey, and objects for interaction are green.
    (Left) Initial multi-modality in the paths that is sampled by the DP. (Right) JM2D rejects unsafe paths using safety filter implicitly throughout the denoising process to guide the generation towards safe and feasible goal-reaching trajectories.}
    \label{fig:zono-rollout}
\end{figure}
\vspace{-0.5em}
\section{Experiments: Constrained Generation}
\label{sec:experiments}
\vspace{-0.5em}

In this section, we illustrate how JM2D overcomes challenges in constrained generation.
Additionally, we assess JM2D's ability to improve conditional generation over robotics baselines.

\subsection{JM2D for Constrained Generation}
\label{subsec: eval for cond gen on toy domain}


\textbf{Setup.}
We now evaluate the efficiency of our algorithm in a conditional generation setting, which we still denote as JM2D for consistency.
We hypothesize that JM2D yields more valid, in-distribution samples than the baselines.
We evaluate JM2D on a modified Donut Toy Domain (\Cref{fig:donut2d_joint}) by removing the model-based variable $\optvar$ and imposing tight constraints that leave only a small feasible region. The task reduces to sampling from the feasible region within the donut while staying in distribution, serving as a proxy for JM2D’s ability to maintain data fidelity while complying with constraints.
We compare against a projection-based conditional diffusion model that enforces feasibility by projecting noisy samples onto the constraint set at each denoising step~\cite{christopher2025constrained}.
We report two metrics: Constraint Alignment (CA), the proportion of valid samples, and Data Fidelity (DF), assessed by alignment with the donut distribution.

\begin{figure}[h]
    \centering
    \begin{subfigure}[b]{0.3\textwidth}
        \centering
        \includegraphics[width=0.8\linewidth]{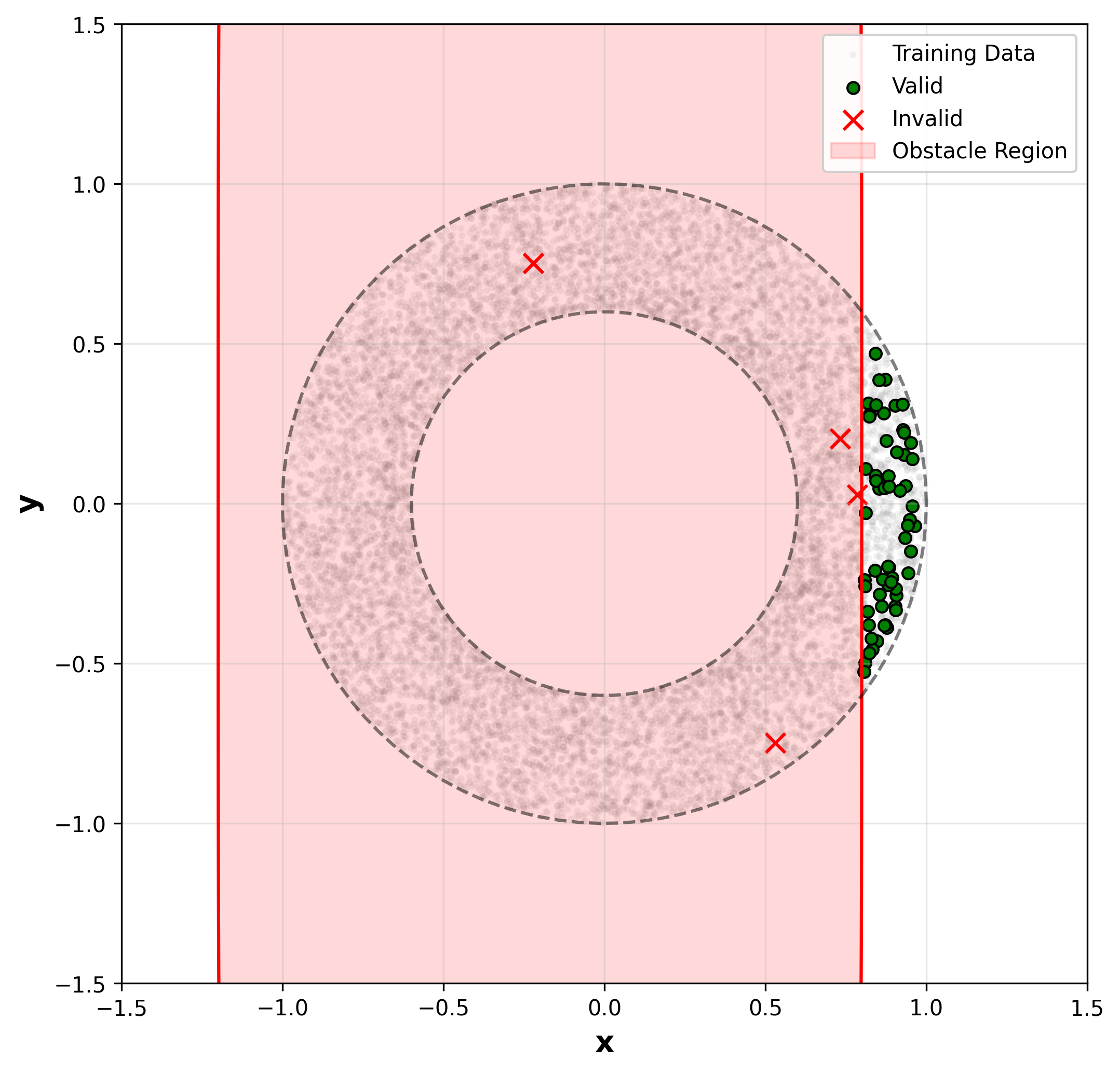}
        \caption{Final samples (ours)}
        \label{fig:donut:ours}
    \end{subfigure}
    \hfill
    \begin{subfigure}[b]{0.3\textwidth}
        \centering
        \includegraphics[width=0.8\linewidth]{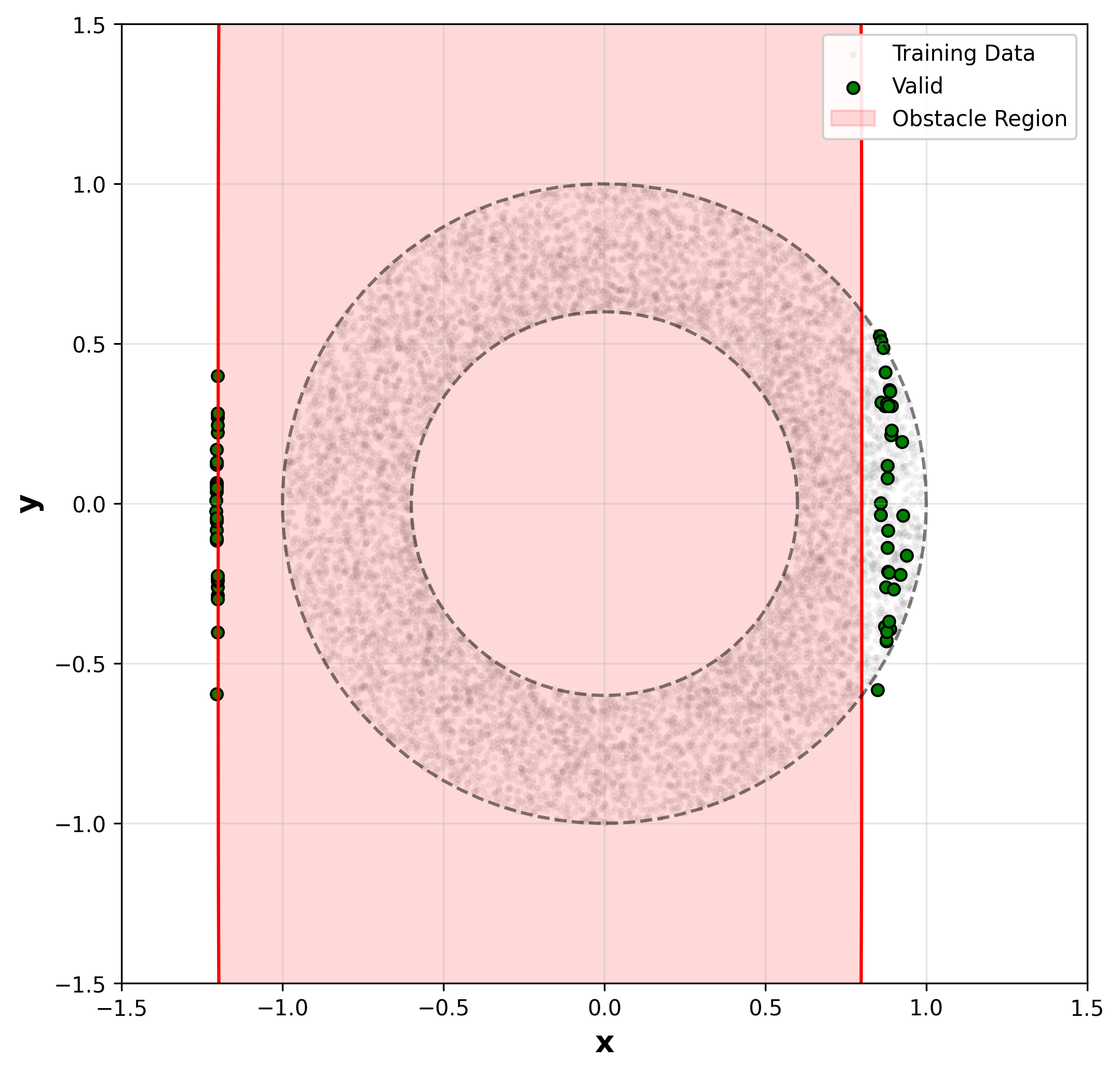}
        \caption{Final samples (projection)}
        \label{fig:donut:proj}
    \end{subfigure}
    \hfill
    \begin{subfigure}[b]{0.35\textwidth}
        \centering
        \includegraphics[width=0.8\linewidth]{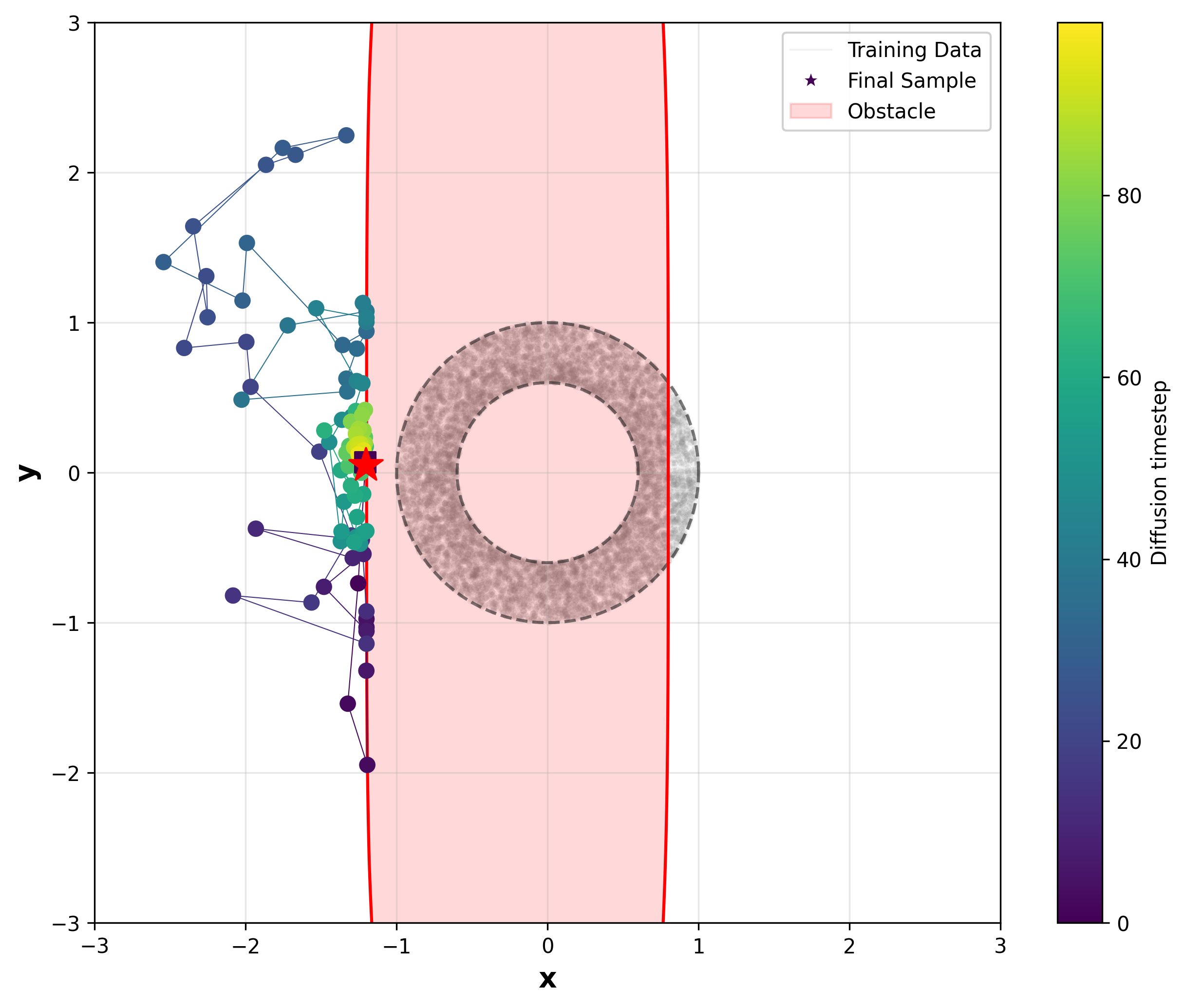}
        \caption{Diffusion trajectory (projection)}
        \label{fig:donut:traj}
    \end{subfigure}
    \caption{
        \textbf{When does projection fail?} 
        For constraint-guided diffusion in the \texttt{Donut} task.
        The gray dots show the training distribution and the red region is a test-time constraint.
        (a) Our method generates 60/64 valid samples within distribution.
        (b) Projection yields valid samples, but 31/64 lie far outside the data distribution.
        (c) Diffusion trajectory of a projected sample: early noise explores an out-of-distribution mode (left), and projection locks it in—despite valid, in-distribution alternatives (right), highlighting how projection can misalign constraint satisfaction with data fidelity.
    }
    \label{fig:donut2d_ours_vs_projection}
    \vspace{-1em}
\end{figure}

\textbf{Result: JM2D Achieves Feasible, In-Distribution Samples.}
As shown in \Cref{fig:donut2d_ours_vs_projection}, JM2D produces many samples that are both feasible and in-distribution, despite not explicitly enforcing constraint satisfaction.
In contrast, while guaranteeing feasibility, projection often generates out-of-distribution samples due to applying constraints on noisy inputs and disregarding the data distribution.

\subsection{JM2D for Conditional Generation of Robot Motion Plans}\label{subsec: eval for cond gen in robotics}

\textbf{Setup.}
Finally, we assess JM2D's ability to solve conditional generation and alignment tasks in robotics, which allows comparison against baselines from the literature and validates critical design choices for sampling with a non-differentiable interaction potential.
We consider the Avoiding environment from D3IL \cite{jia2024towards}, wherein one must generate safe trajectories for a 7DOF manipulator end effector through narrow passages under test-time constraints. The task requires a manipulator to reach a goal region marked by a green line while avoiding static obstacles. Training demonstrations are all collision-free with respect to a fixed set of obstacles (shown in red), ensuring that the learned policy remains in-distribution. The state and action spaces are $s \in \mathbb{R}^4$, consisting of the current and desired end-effector positions, and $a \in \mathbb{R}^2$, denoting Cartesian end-effector velocities. The episode length is capped at 200 timesteps. We train a Diffuser-style planner \cite{janner2022planning} to generate sequences of $(s, a)$ pairs over a fixed horizon of $H = 8$, resulting in $\sample \in \mathbb{R}^{48}$. At test time, novel obstacles (shown in blue) not present during training are introduced, making this an instance of test-time constraint generalization. We evaluate each trained model across 5 random seeds, with 100 trials per seed.

We benchmark against constrained diffusion planners: 
\begin{enumerate}
    \item \textbf{SafeDiffuser}~\cite{xiao2023safediffuser}, which employs invariance principles to ensure constraint satisfaction by design. Specifically, it applies projection at each state without modeling dynamics, leading to the well-documented “trap” behavior.
    \item \textbf{DPCC}~\cite{romer2024diffusion}, which enforces constraint satisfaction via projection onto a feasible set under approximate linear dynamics. Such projection-based diffusion planners enforce constraints by modifying noisy samples post hoc—typically without regard for the learned data distribution or task-specific structure embedded in it. 
This is problematic in settings like robotics, where feasible sets defined by test-time constraints may lie outside the support of the training data (as is often the case in Imitation Learning).
In such cases, projections may find constraint-satisfying solutions that are semantically incorrect, dynamically inconsistent, or simply implausible.
\item \textbf{MPD~\cite{carvalho2024motion}.} Motion Planning Diffusion (MPD) is a method that uses gradients to guide the denoising process. Cost functions are evaluated using noisy samples $\sample_i, i>0$, and the gradient of the cost functions with respect to $\sample_i$ is calculated. The sample is then stepped towards low-cost regions using weight $\lambda$ before the next denoising iteration.
\end{enumerate}


As before, we assess \textit{Safe Success Rate}~(SSucc.), plus \textit{Number of Constraint Violations}~(\#Vio) to assess conservativeness.

\begin{figure}[h]
    \centering
    \begin{subfigure}[b]{0.3\textwidth}
        \centering
        \includegraphics[width=\linewidth]{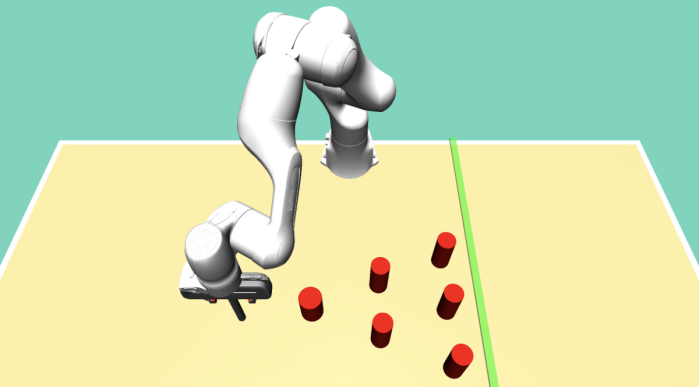}
        \caption{Initial state of \texttt{Avoiding} task}
        \label{fig:subfig1}
    \end{subfigure}
    \hfill
    \begin{subfigure}[b]{0.62\textwidth}
        \centering
        \includegraphics[width=\linewidth]{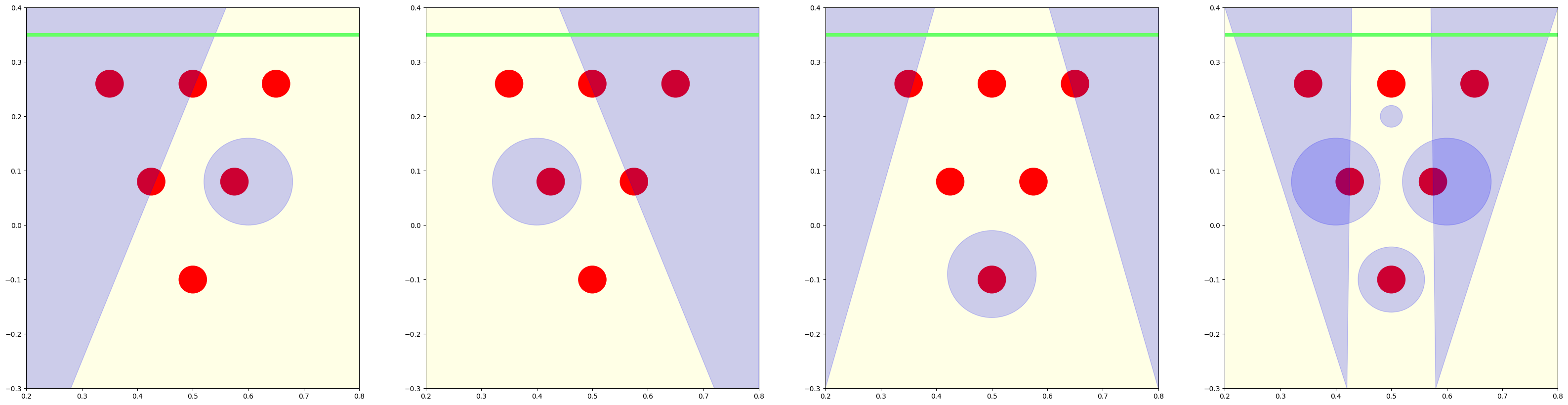}
        \caption{Top-down view of inference-time constraints}
        \label{fig:subfig2}
    \end{subfigure}
    \caption{\textbf{Task description of D3IL-Avoiding \cite{jia2024towards}}.
    (a) The robot must reach the green line without colliding with red obstacles.
    (b) At inference time, additional blue obstacles are introduced as safety constraints. The first three scenarios (Top Left, Top Right, Both Hard) are from the benchmark in DPCC \cite{romer2024diffusion}, and the rightmost scenario (Cluttered) is newly introduced.}
    \label{fig:d3il_avoid_task_desc}
\end{figure}

\textbf{Result: JM2D improves alignment over projection and gradient-based diffusion planners.}
We find that SafeDiffuser’s per‐step projection induces ``trap" behaviors by ignoring system dynamics~\cite{xiao2023safediffuser}, while trajectory‐level projection methods like DPCC rely on linearized dynamics and padded uncertainty margins that overconstrain the planner, failing in cluttered, narrow‐gap maneuvers (\Cref{fig:d3il_cluttered_avoiding_ours_vs_projection}). 
Thus, early enforcement on noisy samples at higher noise levels further degrades trajectory smoothness and fidelity, leading to semantically invalid motions.
Gradient‐based diffusion planners like MPD~\cite{carvalho2024motion} fare no better: by injecting cost gradients into denoising, they introduce noisy corrections that steer samples away from both constraints and the training prior, resulting in unsmooth, out‐of‐distribution trajectories and poor downstream performance (\Cref{tab:main_table}).
JM2D instead integrates constraints into denoising, guiding samples toward in‐distribution feasibility.
This yields dynamically feasible trajectories that satisfy constraints and respect the learned data manifold, improving success rates and constraint violations (\Cref{tab:main_table}).
Please refer to \Cref{app:projection_bad} for discussion.

\begin{table}[h] \small
    \centering
    \vspace{-1em}
    \caption{Performance of all the algorithms in the D3IL-Avoiding task under different constraints.}
    \begin{tabular}{|l|c|c|c|c|c|c|c|c|}
    \hline
       \multirow{2}{*}{Algorithm} & \multicolumn{2}{c|}{Top Left} & \multicolumn{2}{c|}{Top Right} & \multicolumn{2}{c|}{Both} & \multicolumn{2}{c|}{Cluttered}\\
       \cline{2-9} 
       & SSucc. & \#Vio & SSucc. & \#Vio & SSucc. & \#Vio & SSucc. & \#Vio  \\
       \hline 
       DPCC & $\bf 0.87$\tiny{$\pm0.04$} & $\bf0.02$ & $0.72$\tiny{$\pm0.06$} & $0.77$ & $0.66$\tiny{$\pm0.15$} & $1.46$ & $0.0$\tiny{$\pm0.0$} & $6.52$\\
       SafeDiffuser & $0.15$ \tiny{$\pm0.12$} & $13.18$\tiny & $0.07$\tiny{$\pm0.02$} & $18.27$ & $0.01$\tiny{$\pm0.01$} & $11.66$ & $0.01$\tiny{$\pm0.01$} & $22.92$ \\
       MPD & $0.13$\tiny{$\pm0.13$} & $11.97$ & $0.08$\tiny{$\pm0.06$} & $14.85$ & $0.03$\tiny{$\pm0.03$} & $10.80$ & $0.02$\tiny{$\pm0.02$} & $12.71$\\ 
       \ours & $0.77$\tiny{$\pm0.10$} & $1.13$ & $\bf 0.83$\tiny{$\pm0.15$} & $\bf 0.29$ &  $\bf 0.84$\tiny{$\pm0.10$} & $\bf 0.37$ & $\bf 0.75$\tiny{$\pm0.01$} & $\bf 0.32$\\
       \hline
    \end{tabular}
    \label{tab:main_table}
    \vspace{-6pt}
\end{table}

\begin{figure}[h]
    \centering
    \begin{subfigure}[b]{0.3\textwidth}
        \centering
        \includegraphics[width=0.8\linewidth]{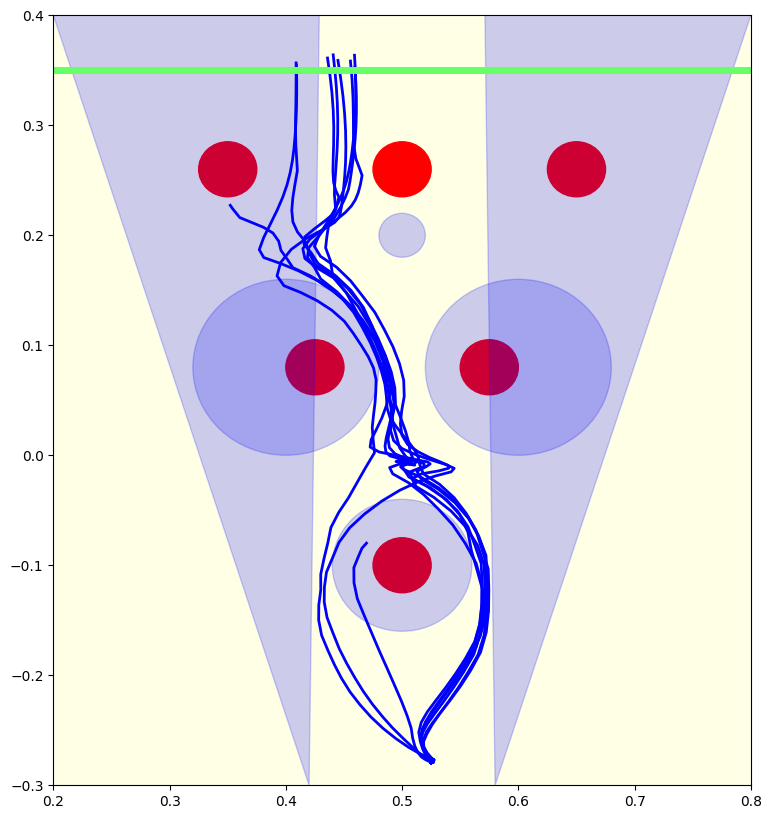}
        \caption{SafeDiffuser \cite{xiao2023safediffuser}}
        \label{fig:cluttered:safediffuser}
    \end{subfigure}
    \hfill
    \begin{subfigure}[b]{0.3\textwidth}
        \centering
        \includegraphics[width=0.8\linewidth]{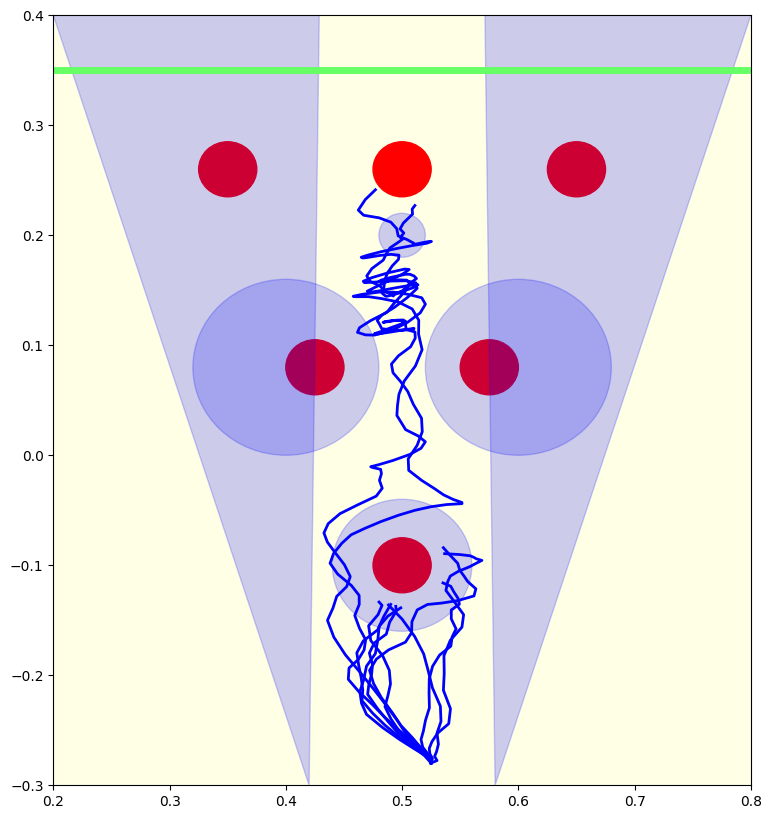}
        \caption{DPCC \cite{romer2024diffusion}}
        \label{fig:cluttered:dpcc}
    \end{subfigure}
    \hfill
    \begin{subfigure}[b]{0.3\textwidth}
        \centering
        \includegraphics[width=0.8\linewidth]{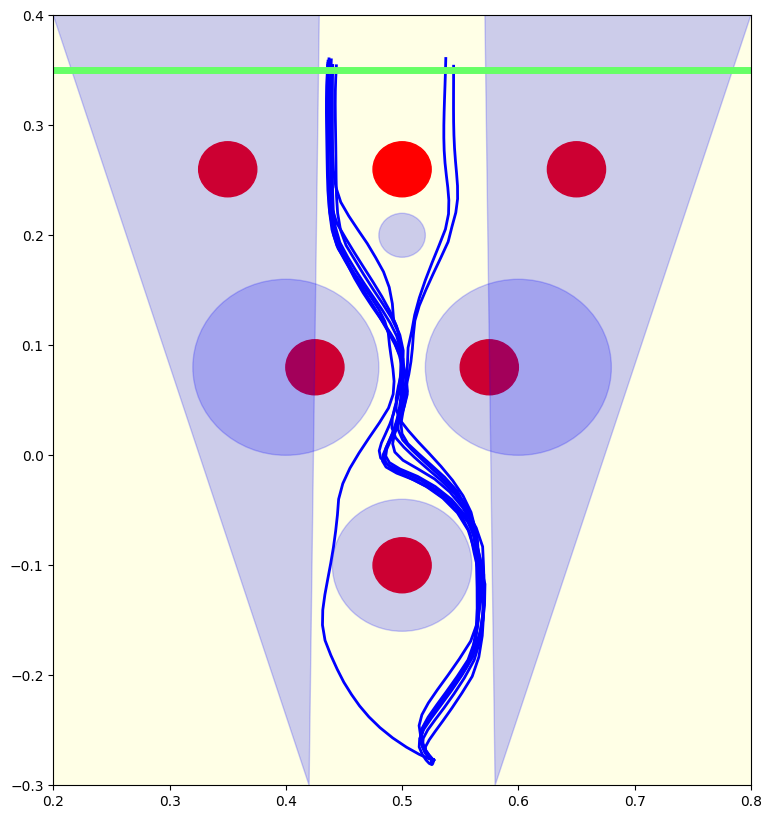}
        \caption{Ours}
        \label{fig:cluttered:ours}
    \end{subfigure}
    \caption{
    \textbf{Planning in Cluttered Environments.}
    In \texttt{Avoiding-Cluttered},
    SafeDiffuser fails due to local traps; DPCC struggles to find a feasible path due to conservative obstacle padding from model uncertainty.
    Our method generates smooth, valid plans without privileged dynamics models.
    }
    \label{fig:d3il_cluttered_avoiding_ours_vs_projection}
\end{figure}

\textbf{Result: Ablation on clean sample estimation quality.} 
As shown in \Cref{algo:mc_sampling} and \eqref{eq:proposal}, our algorithm approximates the denoising score via interaction potentials evaluated on clean samples. Since full denoising is computationally expensive, prior works approximate clean samples using a single-step Tweedie estimate~\cite{carvalho2024motion, daras2024survey} or directly rely on noisy samples. We argue that such approximations can substantially distort generation, particularly under constraints. To assess this effect, we conduct an ablation study on clean-sample estimation and its impact on generation quality.

We evaluate three metrics: Data Fidelity (DF), Objective Alignment (OA; safe success rate), and Constraint Adherence (CA; violation count). DF is measured only in the \texttt{Donut} domain using Chamfer distance between estimated and ground-truth samples. We compare four estimation schemes: (1) noisy samples, (2) single-step Tweedie, (3) $u$-step denoising + Tweedie, and (4) full-step denoising. In (3), $u$ denoising steps are applied at each noise level before Tweedie correction; in (4), samples are fully denoised before interaction potential evaluation.

Results show that in the \texttt{Donut} task, DF increases monotonically with $u$, peaking under full-step denoising. These DF gains align with higher OA and lower CA (\Cref{fig:metrics_vs_Ustep_denoising}(a)). In the robotics \texttt{Avoiding} domain—where DF cannot be directly measured—we observe the same trend: larger $u$ improves OA and reduces CA violations (\Cref{fig:metrics_vs_Ustep_denoising}(b,c)). Overall, our results highlight that effective guidance depends critically on aligning clean-sample estimates with the true support of the target distribution.

\begin{table}[h]
    \centering
    \caption{Performance of all the variants of our algorithm with varying $u$-step denoising for clean sample estimation. Generation quality increases with the number of denoising steps.}
    \begin{tabular}{|c|c|c|c|c|c|c|c|c|}
    \hline
       \multirow{2}{*}{Denoising Step} & \multicolumn{2}{c|}{Top Left} & \multicolumn{2}{c|}{Top Right} & \multicolumn{2}{c|}{Both} & \multicolumn{2}{c|}{Cluttered}\\
       \cline{2-9} 
       & SSucc. & \#Vio & SSucc. & \#Vio & SSucc. & \#Vio & SSucc. & \#Vio  \\
       \hline 
       $u=0$ & $0.13$\tiny{$\pm0.15$} & $12.92$ & $0.06$\tiny{$\pm0.07$} & $18.34$ & $0.0$\tiny{$\pm0.02$} & $12.98$ & $0.01$\tiny{$\pm0.02$} & $23.19$\\
       $u=1$ & $0.47$\tiny{$\pm0.24$} & $3.54$ & $0.33$\tiny{$\pm0.23$} & $5.24$ & $0.32$\tiny{$\pm0.18$} & $5.76$ & $0.34$\tiny{$\pm0.12$} & $1.35$\\
       $u=2$ & $0.50$\tiny{$\pm0.30$} & $3.14$ & $0.45$\tiny{$\pm0.17$} & $2.80$ & $0.50$\tiny{$\pm0.16$} & $3.48$ & $0.39$\tiny{$\pm0.15$} & $1.11$\\
       $u=5$ & $0.68$\tiny{$\pm0.23$} & $1.83$ & $0.65$\tiny{$\pm0.19$} & $1.15$ & $0.82$\tiny{$\pm0.09$} & $0.59$ & $0.63$\tiny{$\pm0.12$} & $0.65$\\
       $u=10$ & $\textbf{0.70}$\tiny{$\pm0.22$} & $\textbf{1.6}$ & $\textbf{0.81}$\tiny{$\pm0.19$} & $\textbf{0.41}$ & $\textbf{0.84}$\tiny{$\pm0.07$} & $\textbf{0.38}$ & $\textbf{0.73}$\tiny{$\pm0.09$} & $\textbf{0.49}$\\
       \hline
    \end{tabular}
    \label{tab:main_table_app}
\end{table}

\begin{figure}[h]
    \centering
    \begin{subfigure}[b]{0.4\textwidth}
        \centering
        \includegraphics[width=\linewidth]{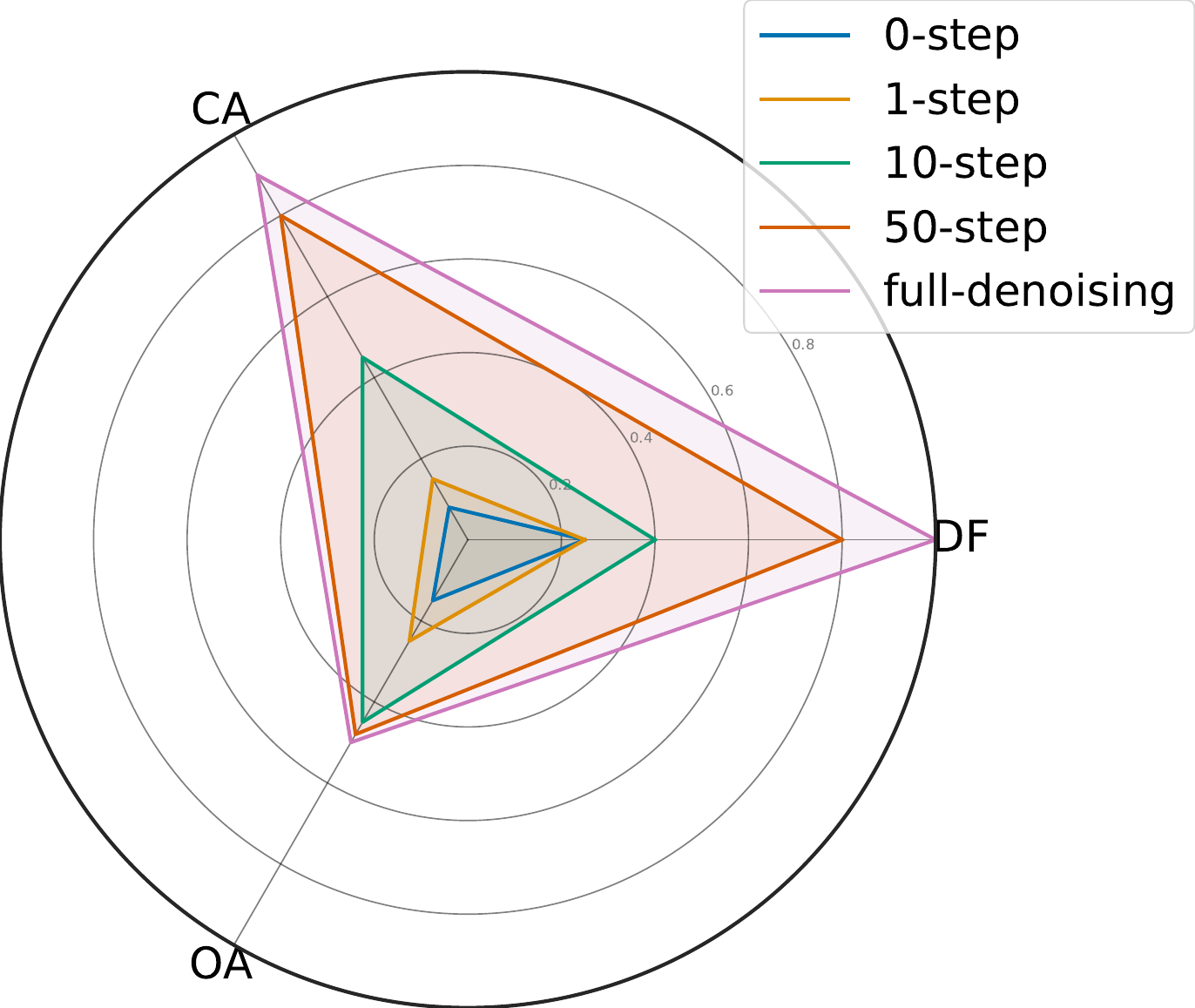}
        \caption{\texttt{Donut}: Metrics vs $u$-step denoising}
        \label{fig:h1:donut}
    \end{subfigure}
    \hfill
    \begin{subfigure}[b]{0.23\textwidth}
        \centering
        \includegraphics[width=\linewidth]{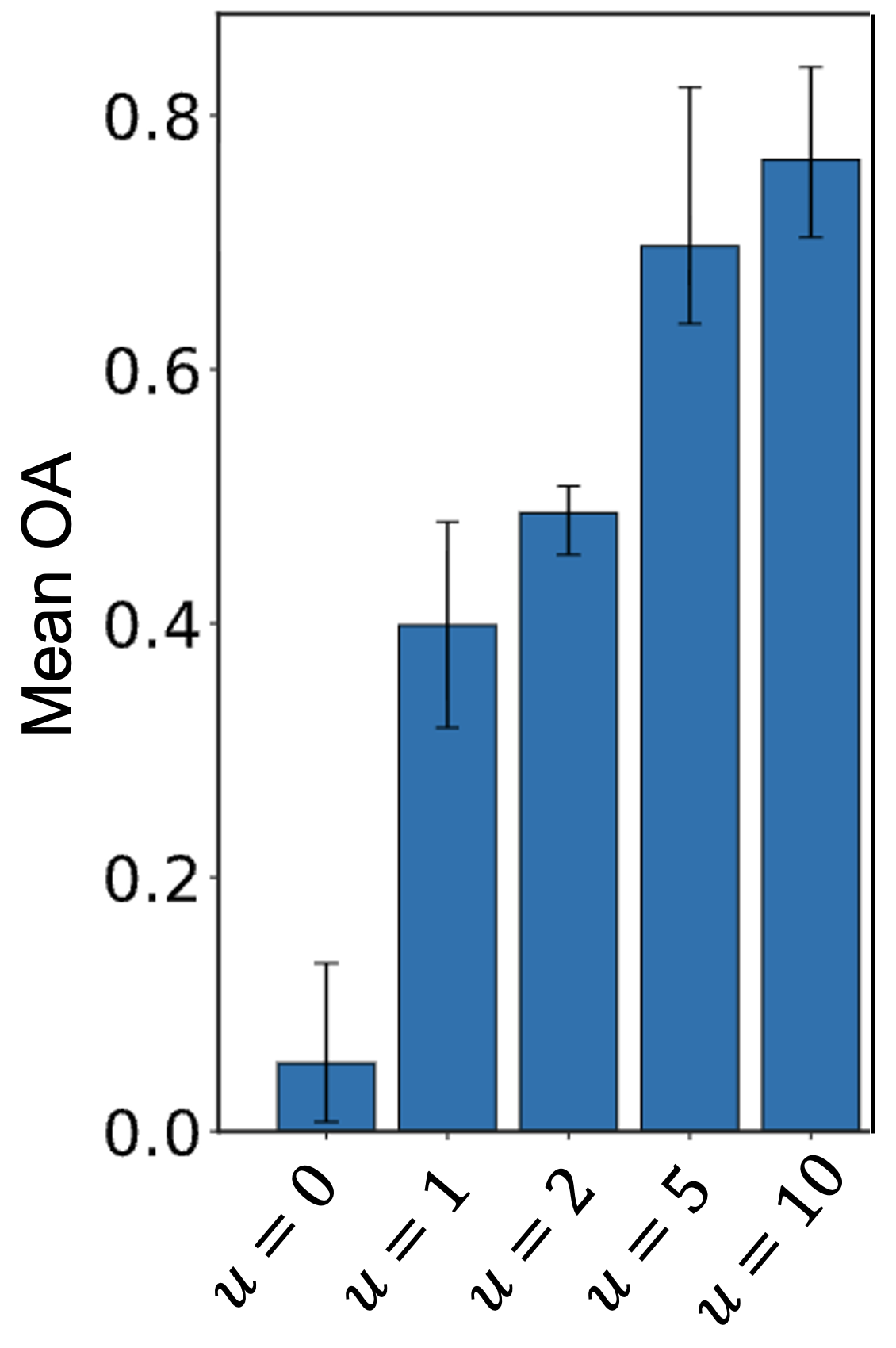}
        \caption{\texttt{Avoiding}: Mean OA vs $u$-step denoising}
        \label{fig:d3il:OA}
    \end{subfigure}
    \hfill
    \begin{subfigure}[b]{0.23\textwidth}
        \centering
        \includegraphics[width=\linewidth]{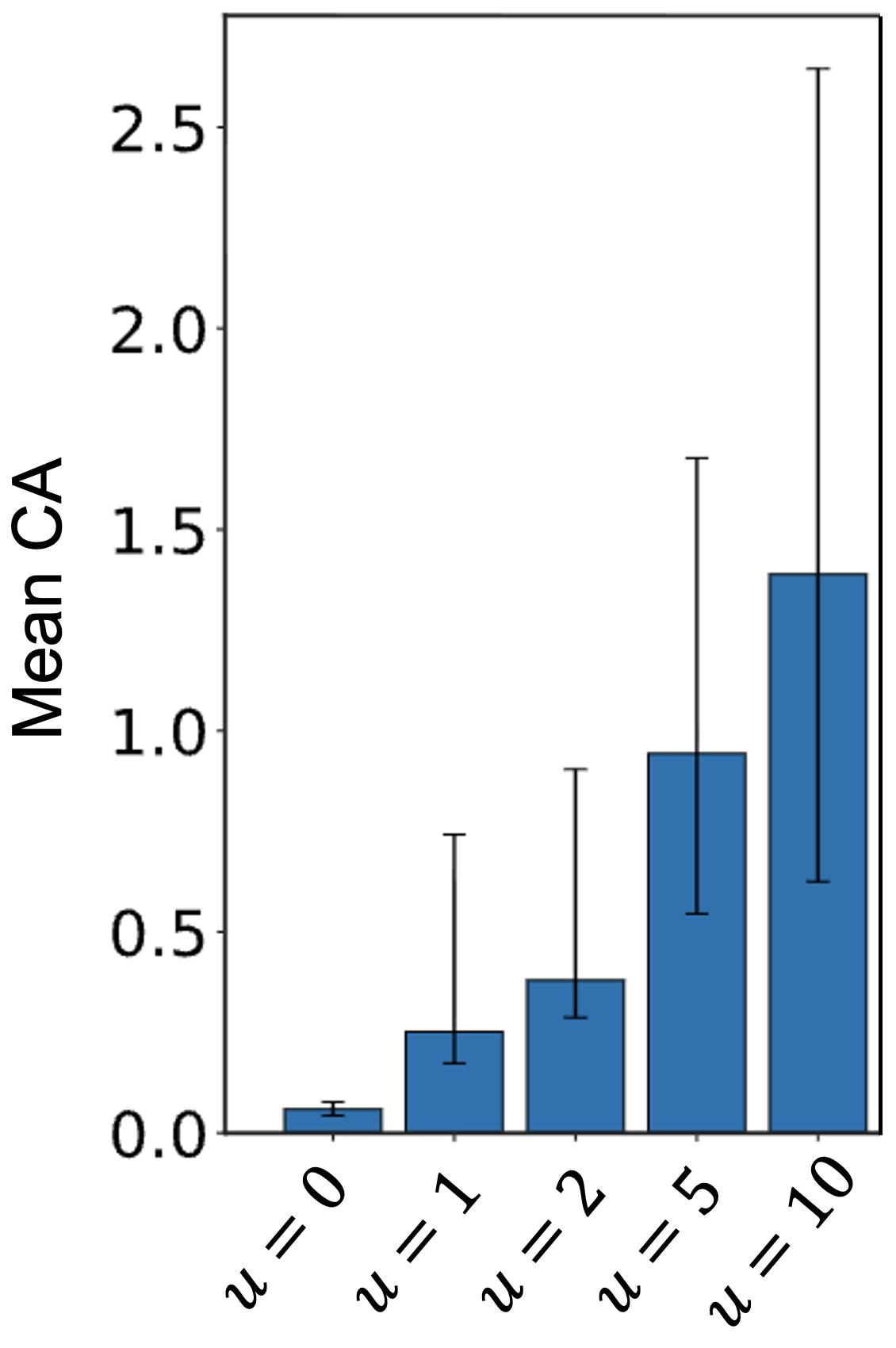}
        \caption{\texttt{Avoiding}: Mean CA vs $u$-step denoising}
        \label{fig:d3il:CA}
    \end{subfigure}
    \caption{
    \textbf{Impact of clean sample quality on guidance effectiveness.}
    (a) On the \texttt{donut} task, increasing denoising steps improves Distribution Fidelity (DF), leading to higher OA and lower CA. (b, c) On the \texttt{Avoiding} task, where DF is not directly measurable, increased denoising similarly improves OA and reduces CA, confirming that better-aligned clean samples enhance guidance.
    }
    \label{fig:metrics_vs_Ustep_denoising}
    \vspace{-1.5em}
\end{figure}
\section{Limitations}
\label{sec:limitations}

Despite the improvements that we found with joint sampling, our framework still suffers from various limitations.
Foremost, we achieve alignment at the cost of increased computational overhead---as we rely on denoising process to construct a Monte-Carlo sample.
Automatically determining when alignment is beneficial remains an open question and is left for future work.
Furthermore, our method assumes that the generative model’s support overlaps the constraint-feasible region; its effectiveness fades when the learned policy becomes less multimodal.
Finally, as our approach cannot strictly enforce constraints, a separate ``backup'' optimization module is still needed to ensure hard-constraint satisfaction.

\section{Conclusion}
\vspace*{-5pt}

We introduced \textit{Joint Model-based Model-free Diffusion (JM2D)}, a unified generative framework that jointly samples from diffusion-based generative models and model-based optimization modules. 
This leads to a novel safety filter that improves performance while maintaining strict safety guarantees compared to traditional decoupled approaches. 
Our experiments show that JM2D performs robustly in realistic robotics tasks with non-differentiable constraints, outperforming conditional diffusion motion planners and conventional safety filters. 
Looking forward, we plan to extend JM2D to other integration problems, such as aligning generative planners with model-based tracking controllers, or aligning multiple robotics modules.


\acknowledgments{
This project was supported by funding from the AI Manufacturing Pilot Facility project under Georgia Artificial Intelligence in Manufacturing (Georgia AIM) from the U.S. Department of Commerce Economic Development Administration, Award 04-79-07808.
}

\bibliography{references}  

\newpage
\appendix
\clearpage
\appendix

\renewcommand\thefigure{\thesection.\arabic{figure}}
\setcounter{figure}{0}    

\renewcommand\thetable{\thesection.\arabic{table}}
\setcounter{table}{0}

\renewcommand\thealgorithm{\thesection.\arabic{algorithm}}
\setcounter{algorithm}{0}

\section{Additional Derivations}
\subsection{Derivation of Eq. \eqref{eq:joint_score_general}} 
\label{app:extended_proof:score_lemma}
We detail the derivation of \eqref{eq:joint_score_general}, which expresses the joint score function prior to applying importance sampling.
\begin{align}
    \nabla_{\jointsample_{i}}\log\prob_{i}(\jointsample_{i})
    &= \frac{
        \nabla_{\jointsample_{i}}\prob_{i}(\jointsample_{i})
    }{
        \prob_{i}(\jointsample_{i})
    } \label{eq:joint_score_1} \\[5pt]
    &= \frac{
        \nabla_{\jointsample_{i}} \int \prob_{\alphajoint_i}(\jointsample_{i} | \jointsample_{0})\prob_{0}(\jointsample_{0}) \, d\jointsample_{0}
    }{
        \prob_{i}(\jointsample_{i})
    } \label{eq:joint_score_2} \\[5pt]
    &= \frac{
        \int \nabla_{\jointsample_{i}} \log\prob_{\alphajoint_i}(\jointsample_{i} | \jointsample_{0})
        \prob_{\alphajoint_i}(\jointsample_{i} | \jointsample_{0})\prob_{0}(\jointsample_{0})
        \, d\jointsample_{0}
    }{
        \prob_{i}(\jointsample_{i})
    } \label{eq:joint_score_3} \\[5pt]
    &= \int \nabla_{\jointsample_{i}} \log\prob_{\alphajoint_i}(\jointsample_{i} | \jointsample_{0}) \, \prob_{0|i}(\jointsample_{0}|\jointsample_{i}) \, d\jointsample_{0} \label{eq:joint_score_4} \\[5pt]
    &= \int s_{\alpha^{\jointsample}_{i}}(\jointsample_{0}|\jointsample_{i}) \, \prob_{0|i}(\jointsample_{0}|\jointsample_{i}) \, d\jointsample_{0} \label{eq:joint_score_5}
\end{align}
We first marginalize the forward perturbation kernel from \eqref{eq:joint_score_1} to \eqref{eq:joint_score_2}, then apply the log-derivative trick from \eqref{eq:joint_score_2} to \eqref{eq:joint_score_3}. Using Bayes’ rule, we obtain \eqref{eq:joint_score_4}, and finally substitute the score term with the predefined Tweedie score $s_{\alpha^{\jointsample}_{i}}(\jointsample_{0}|\jointsample_{i})$ in \eqref{eq:joint_score_5}.
Here, $\prob_{0|i}(\jointsample_{0}|\jointsample_{i})$ is the target distribution for sampling, but it is intractable—motivating the use of importance sampling.

\subsection{Proof of \Cref{thm:main}} \label{app:extended_proof_main_thm}
\begin{proof}
It is sufficient to show that importance weight 
$\frac{\prob_{0|i}(\jointsample_{0}|\jointsample_{i})}{q(\jointsample_{0}|\jointsample_{i})}$ becomes proportional to $\interactionfunc(\sample_{0}, \optvar_{0})\prob^{\optvar}_{0}(\optvar_0)$ under the proposal distribution $q(\jointsample_{0} | \jointsample_{i})$ given in \eqref{eq:proposal}. Since we use a self-normalized importance sampling scheme, any term independent of $\jointsample_0$ cancels out.
Hence we show:
\begin{align*}
\frac{
    \prob_{0|i}(\jointsample_{0} \mid \jointsample_{i})
}{
    q(\jointsample_{0} \mid \jointsample_{i})
}
&=
\frac{
    \prob_{0|i}(\jointsample_{0} \mid \jointsample_{i})
}{
    \prob^{\sample}_{0|i}(\sample_{0} \mid \sample_{i})q^{\optvar}(\optvar_0 \mid \optvar_i)
}
\\
&= 
\frac{
    \prob_{\alphajoint_i}(\jointsample_{i} \mid \jointsample_{0}) \prob_{0}(\jointsample_{0})
}{
    \prob_i(\jointsample_{i})
} \cdot
\frac{
    \prob^{\sample}_{i}(\sample_{i})
}{
    \prob^{\sample}_{\alphasample_i}(\sample_{i} \mid \sample_{0}) \prob^{\sample}_{0}(\sample_{0}) q^{\optvar}(\optvar_{0} \mid \optvar_{i})
}
\\
&= 
\frac{
    \prob^{\sample}_{i}(\sample_{i})
}{
    \prob_i(\jointsample_{i})
} \cdot
\frac{
    \prob_{\alphajoint_i}(\jointsample_{i} \mid \jointsample_{0})
}{
    \prob^{\sample}_{\alphasample_i}(\sample_{i} \mid \sample_{0}) q^{\optvar}(\optvar_{0} \mid \optvar_{i})
} \cdot
\frac{
    \prob_0(\jointsample_0)
}{
    \prob_0^{\sample}(\sample_0)
} \\
&=
\frac{
    \prob^{\sample}_{i}(\sample_{i})
}{
    \prob_i(\jointsample_{i})
} \cdot
\frac{
    \cancel{\prob^{\sample}_{\alphajoint_i}(\sample_{i} \mid \sample_{0})} \prob^{\optvar}_{\alphaoptvar_i}(\optvar_{i} \mid \optvar_{0})
}{
    \cancel{\prob^{\sample}_{\alphasample_i}(\sample_{i} \mid \sample_{0})} q^{\optvar}(\optvar_{0} \mid \optvar_{i})
} \cdot
\frac{
    \cancel{\prob_0^{\sample}(\sample_0)} \prob^{\param}_{0}(\optvar_{0})\interactionfunc(\sample_{0}, \optvar_{0})
}{
    \cancel{\prob_0^{\sample}(\sample_0)}
}\\
&=
Z(\jointsample_{i}, i) \cdot \prob_{0}^{\optvar}(\param_{0})\interactionfunc(\sample_{0}, \param_{0})
\end{align*}
\end{proof}

\subsection{Proof of \Cref{cor:special_case}}
\label{app:extended_proof_cor}

\begin{proof}
From \Cref{thm:main}, we represent the joint score as
\begin{align*}
    \nabla_{\jointsample_{i}} \log \prob_{i}(\jointsample_{i}) 
    = 
    \frac{
        \mathbb{E}_{\hat{\jointsample}_{0} \sim q(\cdot \mid \jointsample_{i})} 
        \left[
            s_{\alphajoint_{i}}(\hat{\jointsample}_0 \mid \jointsample_{i})
            \interactionfunc(\hat{\sample}_{0}, \hat{\optvar}_{0})\prob^{\optvar}_{0}(\hat{\optvar}_{0})
        \right]
    }{
        \mathbb{E}_{\hat{\jointsample}_{0} \sim q(\cdot \mid \jointsample_{i})}
        \left[
            \interactionfunc(\hat{\sample}_{0}, \hat{\optvar}_{0})\prob^{\optvar}_{0}(\hat{\optvar}_0)
        \right]
    }.
\end{align*}

Since the proposal distribution factorizes as 
\[
q(\jointsample_0 \mid \jointsample_{i}) = \prob^{\sample}_{0 \mid i}(\sample_0 \mid \sample_i) \cdot q^{\optvar}(\optvar_0 \mid \optvar_i),
\]
and the interaction potential also factorizes as 
\[
\interactionfunc(\sample_0, \optvar_0) = \interactionfunc^{\sample}(\sample_0) \cdot \interactionfunc^{\optvar}(\optvar_0),
\]
we can decompose the joint score into separate expectations.

\paragraph{Score with respect to $\sample_i$:}
\begin{align*}
    \nabla_{\sample_{i}} \log \prob_{i}(\jointsample_{i}) 
    &= 
    \frac{
        \mathbb{E}_{\hat{\sample}_{0} \sim \prob^{\sample}_{0 \mid i}(\cdot \mid \sample_{i})} 
        \left[
            s_{\alphasample_{i}}(\hat{\sample}_0 \mid \sample_{i}) \interactionfunc^{\sample}(\hat{\sample}_{0})
        \right]
        \cdot
        \cancel{
            \mathbb{E}_{\hat{\optvar}_{0} \sim q^{\optvar}(\cdot \mid \optvar_{i})} 
            \left[
                \interactionfunc^{\optvar}(\hat{\optvar}_{0})
            \right]
        }
    }{
        \mathbb{E}_{\hat{\sample}_{0} \sim \prob^{\sample}_{0 \mid i}(\cdot \mid \sample_{i})} 
        \left[
            \interactionfunc^{\sample}(\hat{\sample}_{0})
        \right]
        \cdot
        \cancel{
            \mathbb{E}_{\hat{\optvar}_{0} \sim q^{\optvar}(\cdot \mid \optvar_{i})} 
            \left[
                \interactionfunc^{\optvar}(\hat{\optvar}_{0})
            \right]
        }
    } \\
    &= 
    \frac{
        \mathbb{E}_{\hat{\sample}_{0} \sim \prob^{\sample}_{0 \mid i}(\cdot \mid \sample_{i})} 
        \left[
            s_{\alphasample_{i}}(\hat{\sample}_0 \mid \sample_{i}) \interactionfunc^{\sample}(\hat{\sample}_{0})
        \right]
    }{
        \mathbb{E}_{\hat{\sample}_{0} \sim \prob^{\sample}_{0 \mid i}(\cdot \mid \sample_{i})} 
        \left[
            \interactionfunc^{\sample}(\hat{\sample}_{0})
        \right]
    }.
\end{align*}

\paragraph{Score with respect to $\optvar_i$:}
\begin{align*}
    \nabla_{\optvar_{i}} \log \prob_{i}(\jointsample_{i}) 
    &= 
    \frac{
        \cancel{
            \mathbb{E}_{\hat{\sample}_{0} \sim \prob^{\sample}_{0 \mid i}(\cdot \mid \sample_{i})} 
            \left[
                \interactionfunc^{\sample}(\hat{\sample}_{0})
            \right]
        }
        \cdot
        \mathbb{E}_{\hat{\optvar}_{0} \sim q^{\optvar}(\cdot \mid \optvar_{i})} 
        \left[
            s_{\alphaoptvar_{i}}(\hat{\optvar}_0 \mid \optvar_{i}) \interactionfunc^{\optvar}(\hat{\optvar}_{0})
        \right]
    }{
        \cancel{
            \mathbb{E}_{\hat{\sample}_{0} \sim \prob^{\sample}_{0 \mid i}(\cdot \mid \sample_{i})} 
            \left[
                \interactionfunc^{\sample}(\hat{\sample}_{0})
            \right]
        }
        \cdot
        \mathbb{E}_{\hat{\optvar}_{0} \sim q^{\optvar}(\cdot \mid \optvar_{i})} 
        \left[
            \interactionfunc^{\optvar}(\hat{\optvar}_{0})
        \right]
    } \\
    &=
    \frac{
        \mathbb{E}_{\hat{\optvar}_{0} \sim q^{\optvar}(\cdot \mid \optvar_{i})} 
        \left[
            s_{\alphaoptvar_{i}}(\hat{\optvar}_0 \mid \optvar_{i}) \interactionfunc^{\optvar}(\hat{\optvar}_{0})
        \right]
    }{
        \mathbb{E}_{\hat{\optvar}_{0} \sim q^{\optvar}(\cdot \mid \optvar_{i})} 
        \left[
            \interactionfunc^{\optvar}(\hat{\optvar}_{0})
        \right]
    }.
\end{align*}

\noindent The score term $\nabla_{\sample_{i}} \log \prob_{i}(\jointsample_{i})$ resembles the conditional guidance used in classifier-based diffusion \cite{huang2024symbolic}. 
Meanwhile, $\nabla_{\optvar_{i}} \log \prob_{i}(\jointsample_{i})$ corresponds to Eq. (9) in \cite{pan2024model}.
\end{proof}

\section{Real world implementation details}
\label{app:realworld}

We provide the sampling hyperparameters of our real-robot setup below:
\begin{table}[h]
    \centering
    \caption{Real-robot JM2D sampling hyperparameters}
    \label{tab:gtamp_training}
    \begin{tabular}{lc}
        \toprule
        \textbf{Hyperparameter} & \textbf{Value} \\
        \midrule
        Denoising process & \texttt{DDPM} \\
        Number of denoising timesteps $I$ & 25 \\
        Number of action samples $N$ & 128 \\
        Number of optimization param samples $N_K$ & 128\\
        Sample selection strategy & \texttt{Rejection Sampling}\\
        Diffusion schedule for action $\alpha^x$ & \texttt{cosine}\\
        Diffusion schedule for optimization parameter $\alpha^k$ & \texttt{linear}\\
        \bottomrule
    \end{tabular}
\end{table}

We use the following public libraries:
\begin{enumerate}
    \item Diffusers: \url{https://huggingface.co/docs/diffusers/en/index}
    \item Robomimic: \url{https://github.com/ARISE-Initiative/robomimic}
    \item Zonopy: \url{https://github.com/roahmlab/zonopy}
\end{enumerate}

\section{Additional experiments: Comparison to Projection}
\label{app:projection_bad}


Here, we discuss how our proposed method addresses the weaknesses of projection methods for balancing the safety-performance tradeoff.
As shown in Section \ref{subsec: eval for cond gen in robotics}, our proposed method enables better alignment between constraints and data fidelity compared to projection-based diffusion planners; in short, JM2D produces smoother and safer trajectories.
This is because JM2D explicitly optimizes the trade-off between hard constraint satisfaction and fidelity to the data-driven priors (e.g., as realistic dynamics or motion patterns) implicitly learned by the generative model.
To proceed, we discuss the limitations of projection through a toy example, and highlight how projection suffers in more complicated examples, negatively impacting data fidelity of trajectories and overall system performance in robotics.

\textbf{Limitations of Projection.}
Projection-based diffusion planners enforce constraints by modifying noisy samples post hoc—typically without regard for the learned data distribution or task-specific structure embedded in it. 
This is problematic in settings like robotics, where feasible sets defined by test-time constraints may lie outside the support of the training data (as is often the case in Imitation Learning).
In such cases, projections may find constraint-satisfying solutions that are semantically incorrect, dynamically inconsistent, or simply implausible.

We illustrate this failure mode in a 2D toy task, where a test-time constraint creates two disjoint feasible regions—only one of which overlaps with the training distribution. 
As shown in \Cref{fig:donut2d_ours_vs_projection}, projection moves samples toward the closest feasible region, regardless of whether the sample remains consistent with the training distribution, leading to low-fidelity samples.
In contrast, our method guides the generative process toward feasible, in-distribution samples.

\begin{figure}[h]
    \centering
    \includegraphics[width=0.8\textwidth]{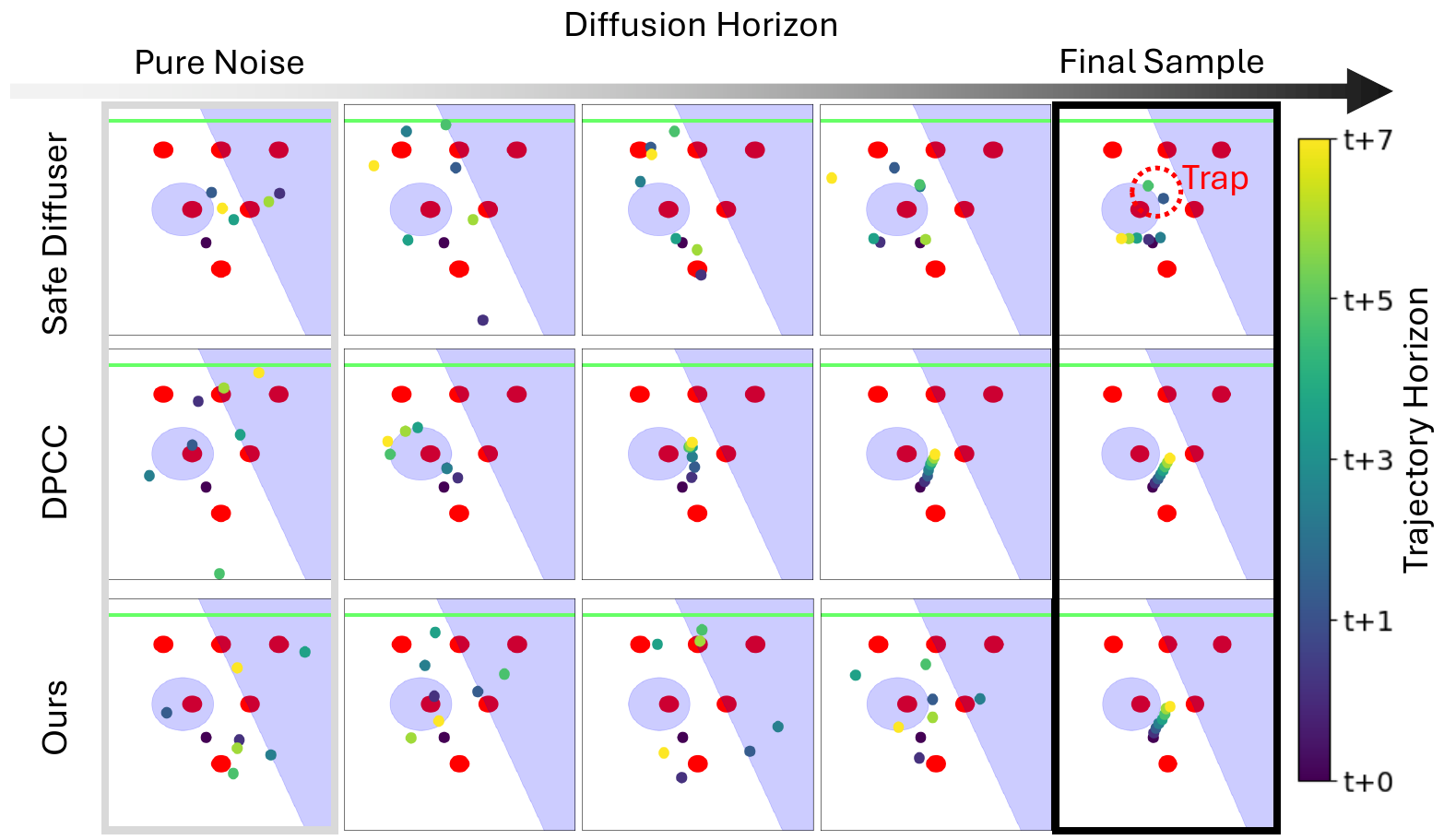}
    \caption{
    \textbf{Denoising behavior in constrained trajectory generation.}
    Visualization of the denoising process across different methods on the \texttt{Avoiding-Left} task.
    Each row shows trajectory evolution from pure noise (left) to final output (right), with color indicating timestep.
    SafeDiffuser projects each state independently, causing trap behavior.
    DPCC imposes dynamics constraints at early stages, which introduces mismatches with training-time data and leads to degraded smoothness.
    Our method preserves randomness and structure across diffusion steps while satisfying constraints.
    }
    \label{fig:d3il_diffusion_comparison}
\end{figure}

\textbf{Effect of Incorporating Dynamics.} 
One of the distinctive data-embedded features in demonstrations is robot dynamics.
In \Cref{fig:d3il_diffusion_comparison}, we visualize how different methods evolve over the diffusion horizon. 
SafeDiffuser applies projection at each state without modeling dynamics, leading to the well-documented “trap” behavior~\cite{xiao2023safediffuser}.
Although trajectory-level projection, such as that in DPCC, introduces a dynamics model to avoid this issue, it does so by injecting hand-designed approximations (e.g., linearized dynamics). 
To account for modeling error, these methods often pad constraints with uncertainty margins, which introduces conservativeness. 
As seen in \Cref{fig:d3il_cluttered_avoiding_ours_vs_projection}, this conservativeness can block feasible plans in cluttered environments requiring narrow-gap maneuvers.

\begin{figure}[h]
    \centering
    \hfill
    \begin{subfigure}[b]{0.25\textwidth}
        \centering
        \includegraphics[width=\linewidth]{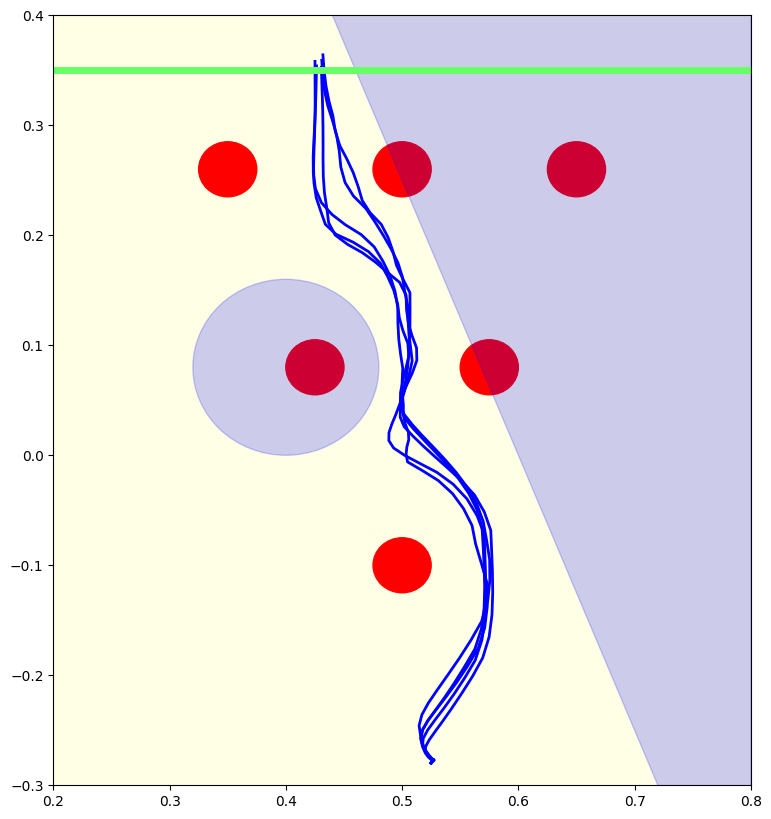}
        \caption{JM2D (\ours)}
        \label{fig:cluttered:safediffuser}
    \end{subfigure}
    \hfill
    \begin{subfigure}[b]{0.25\textwidth}
        \centering
        \includegraphics[width=\linewidth]{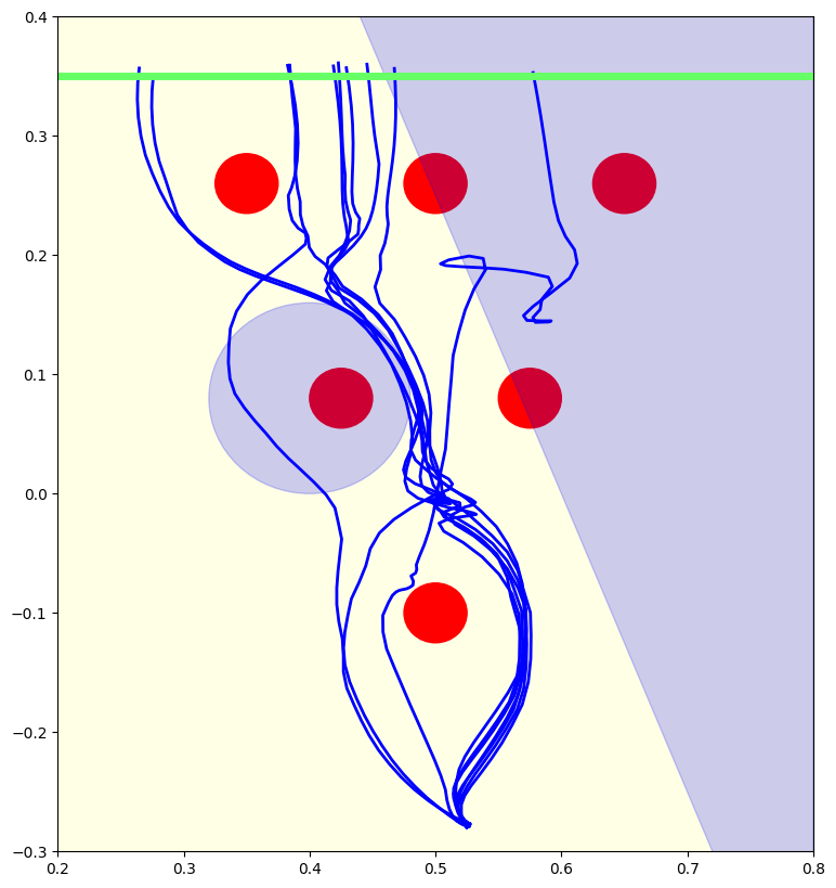}
        \caption{SafeDiffuser \cite{xiao2023safediffuser}}
        \label{fig:cluttered:safediffuser}
    \end{subfigure}
    \hfill
    \begin{subfigure}[b]{0.25\textwidth}
        \centering
        \includegraphics[width=\linewidth]{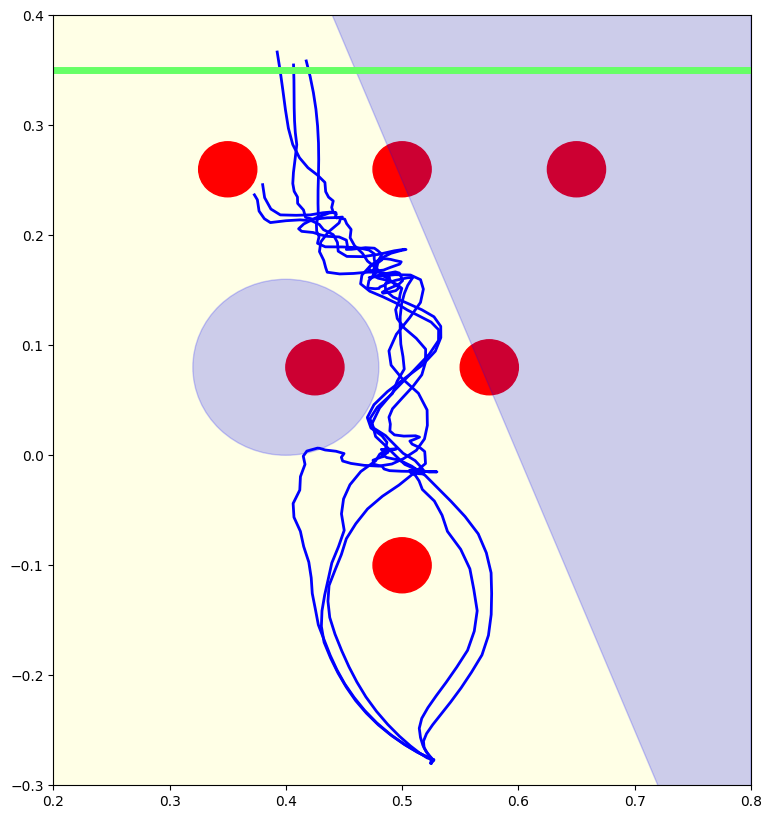}
        \caption{DPCC \cite{romer2024diffusion}}
        \label{fig:cluttered:dpcc}
    \end{subfigure}
    \hfill
    \caption{
    \textbf{Comparison of trajectory smoothness and safety.}
    We compare sampled trajectories from different methods on the \texttt{Avoiding-Left} task. 
    Our method achieves the best balance between data fidelity and constraint satisfaction, resulting in smoother trajectories that avoid collisions.
    }
    \label{fig:d3il_left_hard_smooth}
\end{figure}

\textbf{Mismatched Training Distribution.} 
A subtler failure is shown in \Cref{fig:d3il_diffusion_comparison}, where dynamics-aware projections force early denoising steps to adhere to realistic trajectories. 
However, forward diffusion degrades trajectory structure, and hence noisy samples \textit{should be} dynamically infeasible.
Enforcing dynamics constraints on these corrupted inputs forces the denoiser to operate on unrealistic, out-of-distribution samples. 
This mismatch leads to degraded denoising quality and non-smooth outputs, as visualized in \Cref{fig:d3il_left_hard_smooth}. 
In contrast, JM2D maintains the generative structure across the entire diffusion trajectory, leading to smoother, higher-fidelity results.

\textbf{System Performance.} 
Finally, we argue that although projection-based diffusion motion planners enforce hard constraint satisfaction at each horizon, their tendency to produce low-fidelity trajectories leads to degraded downstream behavior. 
In contrast, while our method does not guarantee per-sample constraint satisfaction, it reduces overall constraint violations by naturally steering the generative process toward safe and plausible solutions. 
This results in improved task performance, as seen in \Cref{tab:main_table}, where our method consistently achieves higher success rates and fewer constraint violations compared to projection-based baselines.

\section{Practical Usage of JM2D}
\label{app:discussion}
We now discuss how JM2D can be adapted to specific user problems of interest. While JM2D is designed for the general structure of such problems, once the user accepts the prioritized assumptions, further refinements are possible as outlined below.

\textbf{Differentiable interaction potential.}
When the interaction potential arises from a differentiable optimization \cite{amos2017optnet, pineda2022theseus}, we can use pathwise gradient estimators \cite{kingma2013auto, ye2024tfg} for lower-variance score estimation.

\textbf{Constraint enforcement strategy.}
If a projection operator onto the feasible set is available, one can project noisy samples back onto the constraint manifold at each diffusion step \cite{christopher2025constrained, romer2024diffusion, bouvier2025ddat}.
We compare these gradient-based and projection-based strategies to our proposed Monte Carlo estimator and show that while privileged methods enforce constraints more aggressively, they tend to excessively perturb the learned distribution compared to our approach empirically (as shown in~\Cref{app:projection_bad}).

\textbf{Design principle of the joint forward noise schedule.} As shown via our formulation, JM2D can use different denoising schedules for model-free and model-based parameter. Our current formulation leverages \texttt{cosine} schedule for the model-free parameter. This aligns with several existing benchmarks of implementing diffusion policies. For the model-based parameter, prior works like MBD~\cite{pan2024model} and DIAL-MPC~\cite{xue2024full} have used \texttt{linear} schedules. JM2D can incorporate the best of both model-free and model-based diffusion sampling.

\textbf{Computational efficiency.} JM2D is based on two main principles:
\begin{enumerate}
    \item Increasing alignment between a predicted clean sample and the interaction potential via complete denoising at every reverse diffusion timestep leads to improved guidance.
    \item Complete denoising at every timestep following a stochastic sampling process like DDPM~\cite{ho2020denoising} leads wide coverage of the training distribution that scales linearly with the number of sampled candidates. Specifically, it relies on the capability of the base diffusion model is capturing multiple principal modes of the training distribution.
\end{enumerate}
Practical application of JM2D is computationally in-efficient. Instead of performing $I$ denoising steps or $I$ evaluations of the score function in standard diffusion sampling, JM2D requires $I(I+1)/2$ evaluations of the score function. Further batch size of JM2D sampling scales with the choice of $N$ and $N_K$ which are critical design parameters. Increasing $N$ and $N_K$ improves performance but increases batch size with $N \times N_K$. As future work, we consider adapting better diffusion samplers with JM2D, particularly:
\begin{enumerate}
    \item DPM /DPM++ solvers: \url{https://github.com/LuChengTHU/dpm-solver}
    \item DEIS: \url{https://github.com/qsh-zh/deis}
\end{enumerate}
This will give us computational flexibility to consider higher values of $N$ and $N_K$.


	
\clearpage

\end{document}